%% file: CAMERA_READY/AISTATS_camera_ready/AISTATS_paper.tex
\documentclass[twoside]{article}

\usepackage[accepted]{aistats2026}

\input{CAMERA_READY/AISTATS_camera_ready/macros}

\begin{document}

%

%

\twocolumn[

\aistatstitle{A Finite Time Analysis of Thompson Sampling for Bayesian Optimization with Preferential Feedback}

\aistatsauthor{ Joseph Lazzaro\And Davide Buffelli \And  Da-shan Shiu \And Sattar Vakili }

\aistatsaddress{ Imperial College London \And MediaTek Research \And MediaTek Research \And MediaTek Research } ]

\begin{abstract}
Preference feedback, in the form of pairwise comparisons rather than scalar scores, has seen increasing use in applications such as human-, laboratory-, and expert-in-the-loop design, as well as scientific discovery. We propose a Thompson Sampling (TS) approach to Bayesian optimization with preferential feedback that models comparisons using a monotone link on latent utility differences and leverages the dueling kernel induced by a base kernel. We provide a finite-time analysis showing that the performance of the proposed method matches that of standard TS for conventional Bayesian optimization with scalar feedback. The analysis exploits the anchor invariance of TS for challenger selection and introduces a double-TS pairing variant. We also demonstrate the performance of the method on both synthetic and real-world examples.
\end{abstract}

\input{CAMERA_READY/AISTATS_camera_ready/1_intro}

\input{CAMERA_READY/AISTATS_camera_ready/2_problem_form}
\input{CAMERA_READY/AISTATS_camera_ready/3_algorithm}
\input{CAMERA_READY/AISTATS_camera_ready/4_analysis}

\input{CAMERA_READY/AISTATS_camera_ready/5_experiments}

\input{CAMERA_READY/AISTATS_camera_ready/6_conclusion}

\bibliography{CAMERA_READY/AISTATS_camera_ready/bibliography}
\bibliographystyle{abbrvnat}

\input{CAMERA_READY/AISTATS_camera_ready/checklist}

\newpage
\onecolumn
\input{CAMERA_READY/AISTATS_camera_ready/appendix}

\end{document}

%% file: CAMERA_READY/AISTATS_camera_ready/macros.tex
\usepackage{amsmath}
\usepackage{amssymb} 
\usepackage{mathtools}
\usepackage{dsfont}
\usepackage{bbm} 
\usepackage{amsthm}
\usepackage{bm}
\usepackage{algorithm}
\usepackage{algorithmic}

\usepackage[authoryear,round]{natbib}

\usepackage[dvipsnames]{xcolor}
\usepackage[
  colorlinks=true,
  citecolor=NavyBlue,
  linkcolor=NavyBlue,
  urlcolor=NavyBlue
]{hyperref}

\DeclareMathOperator*{\argmax}{arg\,max}
\DeclareMathOperator*{\argmin}{arg\,min}
\DeclareMathOperator{\Var}{Var}

\newcommand{\cX}{\mathcal{X}}

\newcommand{\cH}{\mathcal{H}}
\newcommand{\cL}{\mathcal{L}}
\newcommand{\cO}{\mathcal{O}}
\newcommand{\tcO}{\tilde{\mathcal{O}}}
\newcommand{\cF}{\mathcal{F}}

\newcommand{\Hh}{H} 

\newcommand{\Rr}{\mathbb{R}}

\newcommand{\EE}{\mathbb{E}}

\newcommand{\kd}{k^{\Delta}}  
\newcommand{\Kd}{K^{\Delta}} 

\newcommand{\gp}{\text{GP}}

\newtheorem{assumption}{Assumption}
\newtheorem{proposition}{Proposition}
\newtheorem{theorem}{Theorem}
\newtheorem{lemma}{Lemma}

\newtheorem{remark}{Remark}

\usepackage{graphicx}
\usepackage{subcaption}

%% file: CAMERA_READY/AISTATS_camera_ready/1_intro.tex
\section{INTRODUCTION}

Optimizing a black-box objective when only \emph{pairwise preferences} between candidates are available, rather than calibrated scalar evaluations, arises in many modern applications. Examples include expert- and lab-in-the-loop design and evaluation, learning from human feedback, and scientific discovery settings in which absolute scoring is unavailable, too noisy, costly, or ill-defined, but relative judgments are comparatively reliable. These challenges have motivated a line of work on \emph{Bayesian Optimization from Human Feedback} (BOHF), which models preferences via a monotone link applied to latent utility differences, most commonly through the Bradley--Terry--Luce model \citep{bradley1952rank}. BOHF retains the sample-efficiency benefits of kernel methods while operating under strictly weaker feedback than conventional Bayesian optimization. Performance is typically measured by \emph{regret}, which quantifies the loss relative to the globally optimal action across both elements of each queried pair \citep{pasztor2024bandits,xu2024principled,kayal2025bayesian}.

In this paper, we advocate \emph{Thompson Sampling} (TS), which samples from the current posterior and uses the resulting draw as the acquisition function, as a particularly natural rule for BOHF. This contrasts with classic Bayesian optimization acquisition functions such as GP-UCB \citep{srinivas2009gaussian}, which selects points with the largest upper confidence bound; GP-EI \citep{bull2011convergence,wang2014theoretical}, which maximizes expected improvement; and GP-PI \citep{kushner1964new}, which maximizes the probability of improvement. Two properties make TS especially well suited to the preferential setting. First, \emph{anchor independence}: a TS draw of the comparison function decomposes as a difference of latent utilities, so maximizing the draw with respect to the first argument while holding the other fixed is equivalent to maximizing a draw of the latent utility itself. This reduces pair selection to two independent single-argument maximizations without introducing anchor bias. Second, TS naturally supports drawing multiple independent samples from the same posterior surrogate, providing a simple and symmetric mechanism for selecting both elements of a queried pair at each round. These advantages contrast with acquisition functions such as UCB, whose indices couple the two arguments and are deterministic, making it difficult to select two points in a symmetric and fully sequential manner.

TS is a widely used and principled approach for sequential decision-making that naturally balances exploration and exploitation through posterior sampling. The idea is to sample a model from the posterior and select actions that are optimal under that model. This inherent randomness aligns exploration with the agent’s uncertainty and often leads to more efficient learning compared to confidence-bound methods such as UCB~\citep{russo2014learning}. Empirically, TS has been observed to outperform UCB-style acquisitions across a range of bandit and Bayesian optimization problems~\citep{agrawal2012analysis,kaufmann2012thompson}.

Despite TS tracing back to \citet{thompson1933likelihood}, finite-time analyses even for classical bandits emerged only in the 2010s \citep{agrawal2012analysis,kaufmann2012thompson,russo2014learning}. For kernelized settings with scalar feedback, regret guarantees are known for GP-TS variants \citep{chowdhury2017kernelized,vakili2021scalable}. Preference feedback introduces additional challenges: the link nonlinearity and inference from pairwise comparisons, the design of pair-selection mechanisms rather than single-point queries, and the need to relate the resulting choices to standard steps in BO analysis. We develop an analysis of TS tailored to BOHF and show that its regret matches the rates achieved in standard BO with scalar feedback, thereby closing the gap between preferential and scalar-feedback settings. In Theorem~\ref{theorem_regret}, we prove that, with probability at least $1-\delta$, over $T$ rounds the cumulative regret is
\begin{equation}\label{eq:perf_int}
\tcO\!\big(\beta_T(\delta)\,\sqrt{T\,\Gamma(T)}\big), 
\end{equation}
where $\beta_T(\delta) =\cO\left(\sqrt{\Gamma(T)+\log\frac{1}{\delta}}\right)$ is a $(1-\delta)$ confidence-width multiplier and $\Gamma(T)$ is the maximum information gain. This matches the regret bounds established by \citet{chowdhury2017kernelized} for standard TS in the conventional scalar-feedback setting.

Our work relates closely to recent BOHF literature. \citet{pasztor2024bandits} proposed a Stackelberg game to pick the pair (MaxMinLCB). \citet{xu2024principled} presented an optimistic, UCB-based method that fixes one action to one of the points chosen in the previous round. The former attained the same theoretical guarantee as ours given in Equation~\eqref{eq:perf_int}, while the latter achieved a looser guarantee of $\tcO\!\big((T\,\Gamma(T))^{3/4}\big)$. Most recently, \citet{kayal2025bayesian} developed a multi-round, batch exploration scheme that leverages tighter, offline-style confidence intervals to achieve an improved regret bound of $\tcO\!\big(\sqrt{T\,\Gamma(T)}\big)$, at the expense of full sequentiality. In contrast, our approach preserves the simplicity and fully sequential nature of standard acquisition-based BO, selects both points in an independent and symmetric manner, and provides finite-time guarantees that match those for scalar-feedback BO.

\subsection{Contributions}
In this paper, we make the following contributions.
(i) We present a simple TS algorithm for BOHF that exploits anchor independence and multiple posterior samples to select \emph{both} elements of each pair in a \emph{fully sequential and symmetric} manner. Prior BOHF methods are either asymmetric, for example by fixing one action to a previously chosen point, or non-sequential due to batching. Our design proposes two contenders at every round without privileging either side. This is often desirable, and in some cases necessary, when eliciting preferences from human evaluators or large language model (LLM) judges.
(ii) We provide a finite-time analysis establishing regret bounds that match the scalar-feedback setting. We extend the analysis of standard TS to preference feedback, which introduces additional challenges due to the link nonlinearity and the pairwise design. Our proof leverages anchor independence, proven via Proposition~\ref{prop_theta_concentration}, together with novel techniques that relate the symmetrically selected pair to the uncertainty in their utility difference. Combined with information-theoretic bounds based on information gain, this yields $\tilde{\mathcal{O}}\!\big(\beta_T(\delta)\sqrt{T\,\Gamma(T)}\big)$ regret.
(iii) We give a practical implementation recipe using kernel surrogates and validate it on synthetic benchmarks and a real catalyst design task using the OCx24 dataset \citep{abed2024open}, showing competitive or improved performance compared to existing methods.

\subsection{Application and Experimental Setting}

\paragraph{Experimental design and human-in-the-loop.}
In many laboratory and design workflows, absolute measurements are scarce, costly, or poorly calibrated, whereas {relative} judgments are cheap and reliable. For instance, in protein engineering or materials discovery one can often decide which of two candidates yields higher fluorescence, stability, or band gap under identical conditions without running (or trusting) a full calibration pipeline \citep{yang2019machine,dang2025preferential,tucker2020preference}; similar pairwise judgments arise in perceptual tuning and A/B testing. This motivates BOHF framework: a GP surrogate over pairwise differences that preserves BO’s sample-efficiency while operating under strictly weaker, 1-bit feedback~\citep{pasztor2024bandits,xu2024principled,kayal2025bayesian}.

\paragraph{Foundation models as preference oracles.}
Pretrained foundation models, most notably large language models (LLMs), can act as noisy \emph{judges} over pairs by providing preference labels that encode broad domain knowledge \citep{zheng2023judging, lee2024rlaif, liu2024pairs}. This idea is central to learning from AI feedback in alignment and evaluation, where models are trained or assessed using pairwise comparisons rather than scalar scores \citep{christiano2017deep,ouyang2022training,bai2022constitutional,zheng2023judging,liu2023geval}. Combining such an LLM judge with BOHF turns preference-only optimization into a form of \emph{knowledge discovery} from the model. In this setting, the optimizer queries comparative judgments to navigate large design spaces, while the GP surrogate aggregates these feedback signals into principled uncertainty estimates and acquisition decisions.

\subsection{Related Work}

\paragraph{Bayesian Optimization with Scalar Feedback.}
BO is a powerful framework for global optimization of expensive black-box functions, widely used in machine learning, science and engineering applications \citep{wang2023recent}. Classical methods include Expected Improvement \citep{EI_1998,pmlr-v151-tran-the22a}, Upper Confidence Bound \citep{srinivas2009gaussian}, and Thompson Sampling \citep{chowdhury2017kernelized}, which have achieved regret bounds of order $\Gamma(T)\sqrt{T}$, where $\Gamma(T)$ is the maximum information gain after $T$ rounds. Further refinements have achieved either improved regret bounds (e.g., \citealt{whitehouse2024sublinear}) or near-optimal rates, for instance through batched sampling methods \citep{pmlr-v151-li22a}, \emph{Sup} variants of the UCB algorithm~\citep{valko2013finite}, and domain-shrinking methods~\citep{salgia2021domain}, which attain order $\sqrt{\Gamma(T)T}$ matching the lower bounds of \citet{scarlett2017lower}.

\paragraph{Dueling Bandits.}  The dueling bandits problem extends the classical multi-armed bandit framework to settings where feedback is provided in the form of pairwise preferences rather than absolute rewards. In this paradigm, the learner selects two options at each round and receives feedback indicating which is preferred, rather than a numerical score. This approach is particularly relevant in applications where quantifying utility is challenging or expensive, but relative comparisons are feasible or cheap, such as recommendation systems and medical treatment design \citep[for a survey of methods and applications, see][]{bengs2021preference}. Noisy pairwise comparisons between arms are modeled by a preference matrix and the objective is typically defined relative to a Condorcet winner (an arm that beats every other with probability greater than $1/2$) or, when none exists, a Copeland winner (maximizing the number of pairwise wins). Various methods have been proposed \citep[e.g. ][]{yue2012k,pmlr-v32-ailon14,zoghi2015copeland,jamieson2015sparse,ramamohan2016dueling}, including natural optimistic extensions like Relative UCB \citep{pmlr-v32-zoghi14} and Double Thompson Sampling \citep{wu2016double} which have been shown to perform well both numerically and theoretically with logarithmic regret. However, their performance and applicability is highly dependent on the size of the action space being moderately small.  

Methods have also been proposed for larger action spaces. Firstly, linear contextual dueling bandits has been studied, where it is assumed that the arm rewards are linear across the action space \citep[e.g. ][]{dudik2015contextual,saha2022efficient,bengs2022stochastic}. A form of Thompson sampling, called Feel-Good Thompson Sampling, has also been applied to linear contextual dueling bandits \citep{lifeel} and achieved near optimal regret of order $d\sqrt{T}$. For convex structures on the rewards, gradient descent methods have been proposed \citep{yue2009interactively,pmlr-v139-saha21b}. In this paper we will consider much more broad class of reward functions, which we will outline in Section \ref{sec:prelim}.

\paragraph{Bayesian Optimization with Preference Feedback.}

In addition to the three works above that study the same BOHF setting as ours
\citep{xu2024principled,pasztor2024bandits,kayal2025bayesian}, kernel-based learning
from preferential feedback has also been analyzed by \citet{xu2020zeroth} and
\citet{mehta2025sample}, with applications to sample-efficient preference
alignment in large language models. Both works, however, rely on strong
assumptions: they introduce a \emph{Borda} function—the probability that an action
is preferred over a uniformly random alternative—under which the problem
effectively reduces to standard Bayesian optimization.

A closely related setting is also studied by \citet{astudillo2023qeubo}, who
consider Bayesian optimization from preference feedback over finite action
spaces and derive asymptotic, instance-independent upper bounds on the simple
regret. In contrast, we study continuous action spaces and provide
non-asymptotic, instance-dependent regret guarantees.

There have also been a variety of heuristic approaches proposed for BOHF, which
lack both finite-time theoretical guarantees and asymptotic convergence results
\citep{mikkola2020projective,takeno2023towards,siivola2021preferential,gonzalez2017preferential}.

%% file: CAMERA_READY/AISTATS_camera_ready/2_problem_form.tex
\section{PRELIMINARIES AND PROBLEM FORMULATION}\label{sec:prelim}
In this section, we overview the BOHF setup and the prediction and uncertainty tools that underlie our algorithmic design and analysis.

\subsection{BOHF Framework}\label{BOHF}
At each round $t=1,2,\dots,T$, the agent selects a pair $(x_t,x'_t)$ from the action domain $\cX$. The feedback is a binary preference
\[
y_t=\mathds{1}\{x_t \succ x'_t\},\qquad y_t\in\{0,1\},
\]
where $x_t\succ x'_t$ denotes that $x_t$ is preferred to $x'_t$ and $\mathds{1}\{\cdot\}$ is the indicator. Following standard practice, for any pair $(x,x')\in\cX\times\cX$ the random variable $y=\mathds{1}\{x\succ x'\}$ is modeled as
\[
\Pr(y=1\mid x,x')=\mu\big(f(x)-f(x')\big),
\]
with a known, strictly increasing link $\mu:\Rr\to[0,1]$ satisfying $\mu(0)=\tfrac12$; we adopt the logistic (sigmoid) link $\mu(u)=(1+e^{-u})^{-1}$. The unknown latent utility $f:\cX\to\Rr$ quantifies the value of each action. This is the Bradley–Terry–Luce (BTL) model~\citep{bradley1952rank}, widely used for bandits and RL with preference feedback~\citep{pasztor2024bandits,xu2024principled,zhan2023provable,wu2023making, kayal2025bayesian}. When $f(x)>f(x')$ we have $\Pr(x\succ x')>\tfrac12$, and conversely if $f(x)<f(x')$.

This feedback is strictly weaker than in standard BO, where a scalar observation of $f$ is revealed at each step. The objective is to adaptively select pairs so as to approach the globally preferred action
\[
x^\star\in\argmax_{x\in\cX} f(x).
\]
A common performance measure is the cumulative regret over $T$ steps,
\begin{equation}\label{regret_definition_duel}
R(T)=\sum_{t=1}^T \frac{\Pr(x^\star\succ x_t)+\Pr(x^\star\succ x'_t)-1}{2}.
\end{equation}
This is equivalent (up to link-dependent constants) to the utility-based variant $\sum_{t=1}^T \big(f(x^\star)-\tfrac12(f(x_t)+f(x'_t))\big)$ used by~\citet{xu2024principled}; see~\citet{saha2021optimal}.



\subsection{Preliminaries and Assumptions}
As in~\citet{pasztor2024bandits,xu2024principled, kayal2025bayesian}, we assume $f$ lies in the reproducing kernel Hilbert space (RKHS) of a known positive definite kernel $k:\cX\times\cX\to\Rr$. This is a broad assumption: RKHSs of common kernels can approximate essentially all continuous functions on compact subsets of $\Rr^d$~\citep{srinivas2009gaussian}. Let $\cH_k$ be the RKHS of $k$, with inner product $\langle\cdot,\cdot\rangle_{\cH_k}$ and norm $\|\cdot\|_{\cH_k}$. The reproducing property gives $\langle f,k(\cdot,x)\rangle_{\cH_k}=f(x)$. Under mild conditions, Mercer’s theorem yields
\begin{equation}\label{eq:Mercer}
k(x,x')=\sum_{m=1}^{\infty}\gamma_m\,\varphi_m(x)\,\varphi_m(x'),
\end{equation}
with $\gamma_m>0$ and $\{\sqrt{\gamma_m}\varphi_m\}_m$ an orthonormal basis of $\cH_k$. Any $f\in\cH_k$ admits the expansion
$f=\sum_{m=1}^{\infty} w_m\sqrt{\gamma_m}\varphi_m$, with $\|f\|_{\cH_k}^2=\sum_{m} w_m^2$ (details in Appendix~\ref{appendix:background}).

Write $z=(x,x')$ and $h(z)=f(x)-f(x')$ for $(x,x')\in\cX\times\cX$. Following~\citet{pasztor2024bandits}, define the \emph{dueling} kernel
\begin{equation}\label{eq:duel-kernel}
\kd(z_1,z_2)=k(x_1,x_2)+k(x'_1,x'_2)-k(x_1,x'_2)-k(x'_1,x_2),
\end{equation}
for $z_1=(x_1,x'_1)$ and $z_2=(x_2,x'_2)$. One has $\|f\|_{\cH_k}=\|h\|_{\cH_{\kd}}$~\citep[][Prop.~4]{pasztor2024bandits}.

\begin{assumption}\label{ass:RKHS_norm}
The utility $f$ belongs to $\cH_k$ and satisfies $\|f\|_{\cH_k}\le B$ for some $B>0$. Without loss of generality, $k$ is normalized so that $k(\cdot,\cdot)\le 1$ on $\cX\times\cX$.
\end{assumption}

\subsection{Preference Prediction and Uncertainty}\label{prediction_uncertainty}
We briefly recall the standard kernel ridge regression (KRR) predictor for scalar-feedback before turning to preferences. Given scalar data $\{(x_i,o_i)\}_{i=1}^t$ with $o_i=f(x_i)+\varepsilon_i$, KRR yields
\begin{align}
\hat f_t(x)&=k_t^\top(x)\,(K_t+\lambda I)^{-1}\bm o_t,\label{eq:st_krr}\\
k_t(x,x')&=k(x,x')-k_t^\top(x)\,(K_t+\lambda I)^{-1}k_t(x'),\nonumber
\end{align}
where $k_t(x)=[k(x,x_i)]_{i=1}^t$, $K_t=[k(x_i,x_j)]_{i,j=1}^t$, $\lambda>0$, and $\bm o_t=[o_i]_{i=1}^t$. The estimator $\hat f_t$ solves
\[
\hat f_t=\argmin_{g\in\cH_k}\sum_{i=1}^t \big(g(x_i)-o_i\big)^2+\tfrac{\lambda}{2}\|g\|_{\cH_k}^2,
\]
and admits confidence bounds of the form
\[
|f(x)-\hat f_t(x)|\le \hat\beta(\delta)\,\sigma_t(x),
\]
where $\sigma_t^2(x):=k_t(x,x)$ denotes the posterior variance, and $\hat\beta(\delta)$ is a confidence-width multiplier, under standard conditions~\citep{abbasi2013online,chowdhury2017kernelized,vakili2021optimal,whitehouse2024sublinear}.

With preference feedback, there is no closed-form analogue of~\eqref{eq:st_krr}. The setting is classification-like with binary labels. For a history $\Hh_t=\{(x_i,x'_i,y_i)\}_{i=1}^t$, consider the regularized logistic loss over $h\in\cH_{\kd}$:
\begin{align*}
\cL_{\kd}(h,\Hh_t)
&= \sum_{i=1}^t \Big[-y_i\log \mu\big(h(x_i,x'_i)\big)\\
&\hspace{-3em}-(1-y_i)\log\big(1-\mu(h(x_i,x'_i))\big)\Big]+\frac{\lambda}{2}\|h\|_{\cH_{\kd}}^2.
\end{align*}
Define the predictor $h_t$ as the minimizer
\begin{equation}\label{eq:pred}
h_t=\argmin_{h\in\cH_{\kd}}\cL_{\kd}(h,\Hh_t).
\end{equation}
By the representer theorem, $h_t$ admits the finite expansion
\begin{equation}\label{eq:predtheta}
h_t(\cdot)=\sum_{i=1}^t \theta_i\,\kd\big(\cdot,(x_i,x'_i)\big),
\end{equation}
with coefficients $\bm\theta_t=[\theta_1,\dots,\theta_t]^\top\in\Rr^t$. Writing $\kd_t(z)=[\kd(z,(x_j,x'_j))]_{j=1}^t$ and $\kd_t=[\kd((x_i,x'_i),(x_j,x'_j))]_{i,j=1}^t$, the loss in parameter space is
\begin{align}\label{dueling_logistic_loss}
\cL_{\kd}(\bm\theta,\Hh_t)
&=\sum_{i=1}^t \Big[-y_i\log \mu\big(\bm\theta^\top \kd_t(x_i,x'_i)\big)\nonumber\\
&\hspace{-6em}\qquad\quad-(1-y_i)\log\big(1-\mu(\bm\theta^\top \kd_t(x_i,x'_i))\big)\Big]
+\frac{\lambda}{2}\|\bm\theta\|_2^2.
\end{align}

In analogy with~\eqref{eq:st_krr}, an uncertainty proxy is
\begin{equation}\label{std_krr}
\kd_t(z, z')\;=\;\kd(z,z')\;-\;(\kd_t)^\top(z)\,\big(\Kd_t+\lambda\,\kappa\,I\big)^{-1}\,\kd_t(z'),
\end{equation}
where
\begin{equation}\label{eq:kappa}
\kappa \;:=\; \sup_{(x,x')\in\cX^2}\,\frac{1}{\dot\mu\!\big(f(x)-f(x')\big)}
\;=\; \sup_{z\in\cX^2}\,\frac{1}{\dot\mu\!\big(h(z)\big)},
\end{equation}
with $\dot\mu$ the derivative of the link $\mu$. 
We use the notation $\sigma_t^2(z)=\kd_t(z, z)$. 
We formally define the maximum information gain with $t$ preference queries as  
\begin{equation}\label{eq:max_info_gain}
    \Gamma(t) \;=\; \max_{(x_1,x_1'), \dots, (x_t,x_t')} \;
    \tfrac{1}{2} \log \det \!\left(I + (\lambda\kappa)^{-1} K_t^\Delta\right).
\end{equation}
We can then provide the following confidence sequence for our estimate of the difference function $h_t$ from \eqref{eq:pred}.

\begin{lemma}\label{lem:confidence_bound}
    Under Assumption~\ref{ass:RKHS_norm}, with probability at least $1-\delta$,  
    for all $t \geq 1$ and all $x, x' \in \mathcal{X}$,  
    \begin{equation*}
        |\,h_t(x,x') - h(x,x')\,| \;\leq\; \beta_t(\delta)\,\sigma_t(x,x'),
    \end{equation*}
    where
    \begin{equation}\label{eq:beta_def}
        \beta_t(\delta) \;=\; 4B \;+\; 2\sqrt{\tfrac{2\kappa}{\lambda}\,\Big(\Gamma(t) + \log(1/\delta)\Big)}.
    \end{equation}
\end{lemma}

\begin{proof}
This result follows from the proof of \citep[Corollary 5]{pasztor2024bandits}, with the composition by the link function omitted. As a consequence, the Lipschitz constant does not appear in the scaling of $\beta_t(\delta)$.
\end{proof}

%% file: CAMERA_READY/AISTATS_camera_ready/3_algorithm.tex
\section{THOMPSON SAMPLING: FROM SCALAR TO PREFERENTIAL FEEDBACK}
\label{sec:ts_pref}

We describe a TS policy for BOHF, which we refer to as Preferential Feedback TS (PF-TS). The policy uses the pairwise predictor $h_t$ together with the associated uncertainty (covariance) $\kd_t$ from Section~\ref{prediction_uncertainty} (equations~\eqref{eq:pred} and~\eqref{std_krr}) to define a Gaussian {surrogate posterior} on pairs, from which we draw sample paths.

\paragraph{Surrogate posterior on pairs.}
At round $t$, we work with the GP on $\cX\times\cX$ having mean function $h_t$ and covariance $\kd_t$, and draw two \emph{independent} samples
\[
\tilde h_t^{(1)}(\cdot,\cdot),\ \tilde h_t^{(2)}(\cdot,\cdot)\ \sim\ \gp\!\big(h_t,v_t^2\kd_t\big).
\]
To mitigate under-exploration, we include an exploration scale $v_t>0$, as in standard TS for scalar-feedback Bayesian optimization, which multiplies the posterior standard deviation in the sampling step. Its schedule is discussed in Section~\ref{sec:analysis}, and in particular we will set $v_t = \beta_t(\delta)$ in our analysis.

\begin{algorithm}[h]
\caption{Preferential Feedback Thompson Sampling (PF-TS) }
\label{algo:thompson}
\begin{algorithmic}[1]
\REQUIRE Horizon $T$
\STATE Initialize $t\gets 1$.
\WHILE{$t \le T$}
    \STATE Compute $h_t$ and $\kd_t$ according to \eqref{eq:pred} and \eqref{std_krr}.
    \STATE Sample $\tilde h_t^{(1)}(\cdot,\cdot) \sim \gp\!\big(h_t,v_t^2\kd_t\big)$.
    \STATE Sample $\tilde h_t^{(2)}(\cdot,\cdot) \sim \gp\!\big(h_t,v_t^2\kd_t\big)$.
    \STATE Pick an anchor $x_0\in\cX$.
    \STATE Select
    $
        x_t  \in \argmax_{x\in\cX}\ \tilde h_t^{(1)}(x,x_0)$, \\
    \STATE Select $
        x'_t \in \argmax_{x\in\cX}\ \tilde h_t^{(2)}(x,x_0).
    $
    \STATE Query $(x_t,x'_t)$, observe preference $y_t\in\{0,1\}$, and augment the data.
    \STATE Update $t\gets t+1$.
\ENDWHILE
\end{algorithmic}
\end{algorithm}

\subsection{Anchor independence and reduction to single-argument maximization.}\label{subsection:anchor_indep}
Fix any anchor $x_0\in\cX$ as a device to score candidates via $\tilde h_t^{(i)}(\cdot,x_0)$. Because the comparison function is a difference,
$
h(x,x')=f(x)-f(x'),
$
maximizing over $x$ yields
\[
x^\star \;\in\; \argmax_{x\in\cX}\, h(x,x_0)
\;=\; \argmax_{x\in\cX}\, f(x),
\]
so the maximizer is independent of the anchor $x_0$. In Section~\ref{subsection:anchor_independence} we show that any posterior Thompson sample $\tilde h_t$ is also separable, i.e., $\tilde h_t(x,x')=\tilde f_t(x)-\tilde f_t(x')$, and the same argument applies. Hence, the algorithm’s performance is \emph{anchor-independent}; any choice of $x_0$ yields the same selected actions.

\subsection{Action selection}\label{subsection:action_selection}
The pair at round $t$ is chosen by a \emph{double} Thompson Sampling procedure from two independently sampled difference functions. Draw two independent surrogates $\tilde h_t^{(1)},\tilde h_t^{(2)} \sim \gp(h_t,v_t^2\kd_t)$ and fix an anchor $x_0\in\cX$. Select
\[
x_t \in \argmax_{x\in\mathcal{X}}\ \tilde h_t^{(1)}(x,x_0),
\qquad
x'_t \in \argmax_{x\in\mathcal{X}}\ \tilde h_t^{(2)}(x,x_0).
\]
Finally, play the pair $(x_t,x'_t)$, observe $y_t$, and recompute $h_{t+1}$ and $\kd_{t+1}$ via~\eqref{eq:pred} and~\eqref{std_krr}.

%% file: CAMERA_READY/AISTATS_camera_ready/4_analysis.tex
\section{ANALYSIS}\label{sec:analysis}

In this section, we analyze the performance of the PF-TS algorithm.

\subsection{Anchor-independence}\label{subsection:anchor_independence}

A key property that makes TS naturally extendable to the BOHF setting is the {anchor-independence} of the acquisition (TS samples), in contrast to other acquisitions such as UCB. We formalize this below.

\begin{proposition} \label{prop_theta_concentration}
Let $\kd$ be a dueling kernel and let $\Hh_t=\{(x_i,x'_i,y_i)\}_{i=1}^t$ be the observation set as in Section~\ref{BOHF}. Let $h_t$ and $\kd_t$ be the predictor and uncertainty estimates defined via~\eqref{eq:pred} and~\eqref{std_krr}. Let $\tilde h_t \sim \gp(h_t,v_t^2\kd_t)$ be a Thompson sample on $\cX^2$, where $v_t$ is a fixed sequence. Then there exists a function $\tilde f_t$ such that, for all $(x,x')\in\cX^2$, almost surely,
\begin{equation}
\tilde h_t(x,x') \;=\; \tilde f_t(x)\;-\;\tilde f_t(x').
\end{equation}
\end{proposition}

Here, $\tilde f_t$ can be interpreted as a posterior sample of the latent utility $f$. A detailed proof is provided in Appendix~\ref{appendix:anchor}.

This contrasts with the UCB index, which is not separable in this way. As a result, GP-UCB is not anchor-independent: the choices of $x_t$ and $x'_t$ generally depend on the anchor, and the reduction used for TS does not apply. Figure~\ref{fig:anchor} illustrates this by plotting the argmax of the first argument ($x_t$) while fixing the second argument to an anchor point $x_0$. We compare UCB acquisitions (constructed using Lemma~\ref{lem:confidence_bound}) with Thompson samples for different anchors $x_0$. While $x_t$ is invariant to the anchor under TS, it clearly depends on the anchor under UCB. 

The anchor-independence property is also practically relevant. When selecting an action by maximizing the acquisition function, as described in Section~\ref{subsection:action_selection}, it suffices to optimize over a single parameter at a time. This stands in contrast to prior BOHF methods, which require searching over pairs of parameters and therefore incur a higher computational burden. In our current implementations, we perform this optimization by searching over a discrete set of candidate parameters; further details are provided in Appendices~\ref{app:theorem_proof} and~\ref{appendix:experiment}.

Bayesian optimization settings often assume that data acquisition, such as querying a human or model-based preference oracle, is the dominant cost. As a result, computational savings in optimizing the acquisition function are typically secondary to sample efficiency, which is the primary focus of this paper. Nevertheless, the uncertainty estimates used by our method can themselves be approximated. For instance, posterior variances may be approximated via Nyström methods applied to the kernel matrix inverse. Such approximations preserve the regret guarantees, provided they only affect the exploration scale up to constant factors, as established in related settings by \citet{vakili2022improved}.

\begin{figure}[t]
    \centering
    \includegraphics[width=0.8\linewidth]{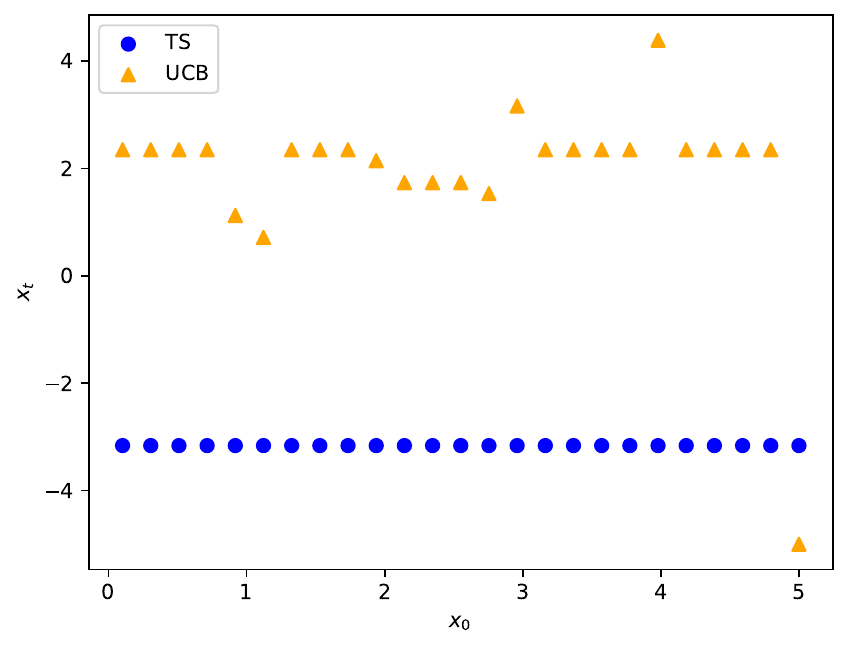}
    \caption{Selected $x_t$ under TS (blue) and UCB (orange) as the anchor $x_0$ varies. Utility is the Ackley function; the history contains 200 uniformly sampled pairs. TS is anchor invariant: the chosen $x_t$ does not change with $x_0$, whereas the point chosen by UCB depends on the anchor.}
    \label{fig:anchor}
\end{figure}

\subsection{The performance analysis of PF-TS}

In this section, we provide a finite-time analysis of the PF-TS algorithm. We recover the same guarantees as vanilla TS with scalar-valued feedback, despite operating under a more restrictive feedback model. In particular: (i) preference feedback supplies only minimal, 1-bit information per query; (ii) unlike vanilla TS, our selection uses an anchor point, whose effect we handle via the separability (anchor-independence) property; and (iii) we select the pair by drawing two independent Thompson samples and maximizing against the same anchor, while the observed label depends on the {pair}, necessitating arguments different from standard TS in conventional BO. Our proof develops techniques that relate the symmetrically selected pair to the posterior uncertainty of their utility difference and leverages bounds based on information gain.

\begin{theorem} \label{theorem_regret}
Consider the BOHF problem introduced in Section~\ref{BOHF}. Under Assumption~\ref{ass:RKHS_norm}, the performance of the PF-TS given in Algorithm~\ref{algo:thompson}, with $v_t=\beta_t(\delta)$, satisfies
\begin{equation}
    R(T) \;=\; \Tilde{\mathcal{O}}\!\left(\beta_T(\delta)\,\sqrt{T\,\Gamma(T)}\right),
\end{equation}
with probability greater than $1-2\delta$. Here $\Gamma(T)$ denotes the maximum information gain under the dueling kernel $\kd$ and we have omitted logarithmic terms.
\end{theorem}

\paragraph{Proof sketch.}
On a high-probability event (uniform in $t$), both the surrogate mean $h_{t-1}$ and the two TS draws $\tilde h_{t}^{(1)},\tilde h_{t}^{(2)}$ lie within $\beta_t(\delta)$ standard deviations of the true $h$. Using separability $h(x,x')=f(x)-f(x')$ and any anchor $x_0$, write $h(x^\star,x_t)=h(x^\star,x_0)-h(x_t,x_0)$, add and subtract the TS draw used to pick $x_t$, and use that it maximizes $x\mapsto \tilde h_t^{(1)}(x,x_0)$. Choosing the anchor by \emph{anchor independence} ($x_0=x'_t$ for the first draw and $x_0=x_t$ for the second) yields the per-round utility regret bounded by (up to constants)
\begin{align*}
\beta_t(\delta)\sqrt{\log\!\left(\tfrac{t}{\delta}\right)}\Big[\sigma_{t-1}(x_t,x'_t)\;&+\;\sigma_{t-1}(x^\star,x_t)\;\\
&~~~~~~~~~~~~ +\;\sigma_{t-1}(x^\star,x'_t)\Big],    
\end{align*}
where the Lipschitzness of the link converts utility gaps to probability gaps in~\eqref{regret_definition_duel}. 

Intuitively, the terms involving $x^\star$ are small: as the algorithm explores, points near the maximizer are repeatedly compared, so the posterior uncertainty for pairs comparing $x^\star$ shrinks quickly. We prove this by extending standard TS tools: an anti-concentration argument ensures that, with constant probability each round, the TS draw picks a point whose variance (relative to the other action as anchor) is not overly small; this lets us {link} the $x^\star$-terms to the variance at the played pair and bound
$
\sigma_{t-1}(x^\star,x_t)+\sigma_{t-1}(x^\star,x'_t)
$
in expectation by
$
\sigma_{t-1}(x_t,x'_t)
$
(up to constants). Finally, the sum of sequential standard deviations at the queried pairs is controlled by an information-gain bound (analogous to the elliptical potential lemma in linear bandits), yielding
\[
\sum_{t=1}^T \sigma_{t-1}(x_t,x'_t)\;=\;\tilde{\mathcal{O}}\!\big(\sqrt{T\,\Gamma(T)}\big),
\]
and thus the overall regret rate.

\begin{remark}
The quantity $\Gamma(T)$ is a kernel-dependent, algorithm-independent complexity term that appears throughout BO and BOHF analyses (see, e.g., \citealp{srinivas2009gaussian,pasztor2024bandits,xu2024principled}). Its growth with $T$ is well understood for common kernels: for linear kernels, $\Gamma(T)=\mathcal{O}(d\log T)$; for kernels with exponentially decaying Mercer eigenvalues (e.g., Squared Exponential), $\Gamma(T)=\mathrm{polylog}(T)$; and for Matérn kernels with smoothness $\nu>1/2$ in dimension $d$, $\Gamma(T)=\tilde{\mathcal{O}}\!\big(T^{\,d/(2\nu+d)}\big)$ \citep[see, e.g.,][]{vakili2021optimal}. Furthermore, \citet[Prop.~4]{pasztor2024bandits} show that the dueling kernel has eigenvalues exactly twice those of the base kernel, so the \emph{scaling} of $\Gamma(T)$ with $T$ is identical for the dueling and base kernels. Consequently, our bound attains the same regret order as GP-TS in conventional BO with scalar feedback.
\end{remark}

%% file: CAMERA_READY/AISTATS_camera_ready/5_experiments.tex
\section{EXPERIMENTS}\label{sec:experiments}

In this section we illustrate the numerical performance of PF-TS. In the interest of space, we leave some additional experiments to be presented in Appendix \ref{appendix:additional_experiments}.

\begin{figure*}[ht]
\centering
\begin{subfigure}{0.49\linewidth}
\centering
    \includegraphics[width=0.7\linewidth]{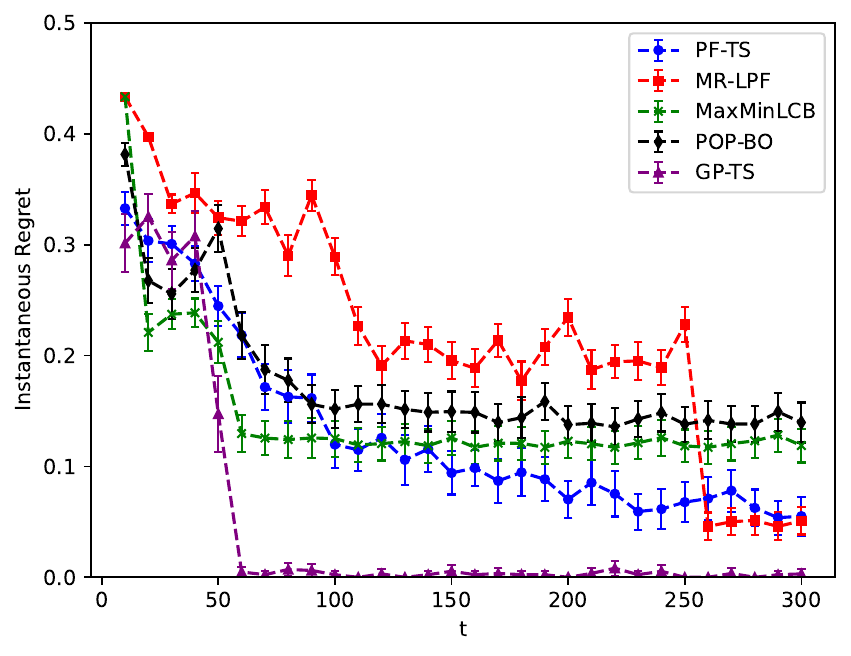}
        \caption{\textbf{(Ackley)} Instantaneous Regret}
    \label{fig:ackley_1}
\end{subfigure}
\begin{subfigure}{0.49\linewidth}
\centering
    \includegraphics[width=0.7\linewidth]{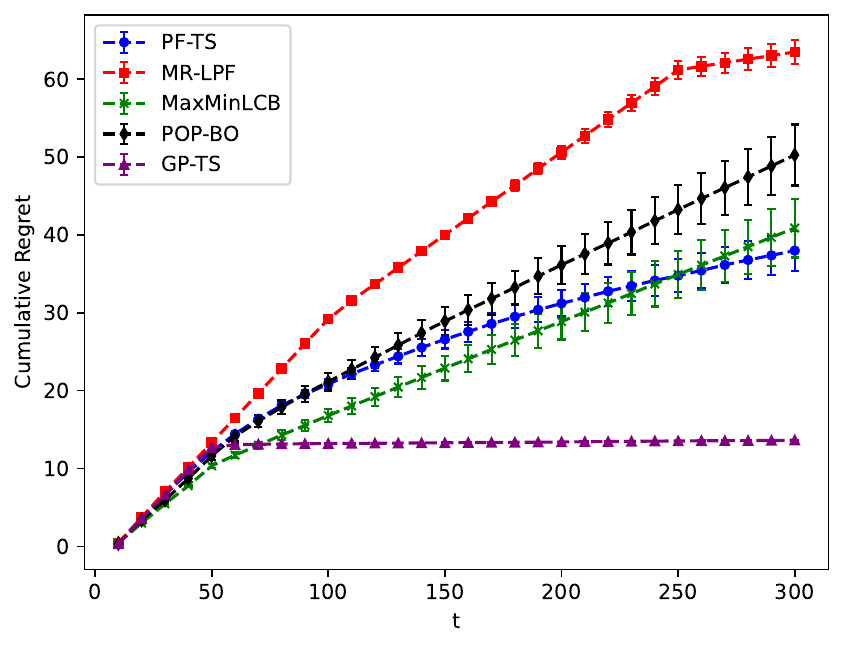}
        \caption{\textbf{(Ackley)} Cumulative Regret}
    \label{fig:ackley_2}
\end{subfigure}
\begin{subfigure}{0.49\linewidth}
\centering
    \includegraphics[width=0.7\linewidth]{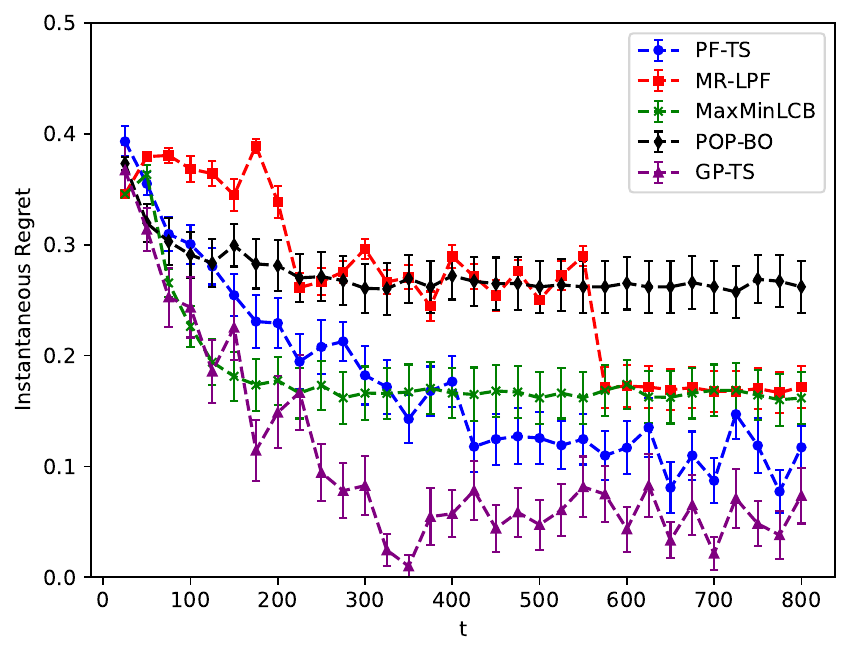}
        \caption{\textbf{(Catalyst)} Instantaneous Regret}
    \label{fig:cat_1}
\end{subfigure}
\begin{subfigure}{0.49\linewidth}
\centering
    \includegraphics[width=0.7\linewidth]{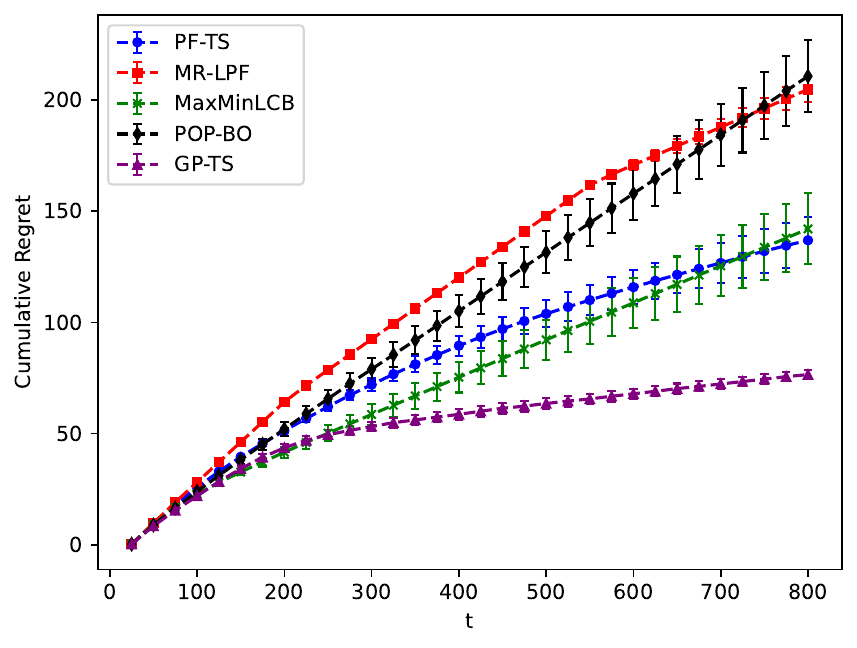}
        \caption{\textbf{(Catalyst)} Cumulative Regret}
    \label{fig:cat_2}
\end{subfigure}
\caption{Regret of PF-TS, MR-LPF, MaxMinLCB, POP-BO, and GP-TS. Mean over 30 runs; bands show $\pm$1 standard error. (a,b) Ackley objective ($T=300$). (c,d) Catalyst hydrogen-yield objective ($T=800$). Left column: instantaneous (simple) regret. Right column: cumulative regret~\eqref{regret_definition_duel}.}

\label{fig:cat_combined}
\vskip -0.2in
\end{figure*}

\subsection{Comparison with Existing Methods}\label{subsection:ackley_experiments}
We start by considering a synthetic utility function $f$ given by the one-dimensional Ackley function, a nonconvex benchmark with many local optima, widely used to test Bayesian optimization algorithms \citep{jamil2013literature} and adopted in recent BOHF studies \citep{pasztor2024bandits,kayal2025bayesian}. In Figures~\ref{fig:ackley_1} and~\ref{fig:ackley_2}, we plot the regret for PF-TS (ours), MR-LPF \citep{kayal2025bayesian}, MaxMinLCB \citep{pasztor2024bandits}, and POP-BO \citep{xu2024principled}. We use the same inference procedure as in Section~\ref{prediction_uncertainty} and a Mat\'ern-$5/2$ kernel for all methods. Figure~\ref{fig:ackley_1} reports the mean instantaneous regret, and Figure~\ref{fig:ackley_2} the mean cumulative regret (with standard errors), over 30 runs up to $T=300$. For comparison, we also include vanilla 
GP-TS~\citep{chowdhury2017kernelized}, which receives direct scalar feedback on the same $f$ (with standard Gaussian observation noise). To place them on the same regret scale as the BOHF methods, we evaluate their regret by repeating the same action at each time $t$, i.e., $x'_t = x_t$.

From Figure~\ref{fig:ackley_2}, the cumulative regret of our method after 300 iterations is significantly smaller than that of MR-LPF and POP-BO, and competitive with MaxMinLCB. We also see in Figure~\ref{fig:ackley_1} that the instantaneous (simple) regret achieved by our method at $T=300$ is lower than that of MaxMinLCB. By design, MR-LPF attains competitive instantaneous regret at this horizon, consistent with its asymptotically optimal regret order \citep{kayal2025bayesian}. Note, the cumulative regret curve appears smoother than the instantaneous regret curve because it aggregates the instantaneous values over time, averaging out short-term fluctuations and noise.

\subsection{Catalyst Design with Preference Feedback (OCx24)}\label{subsection:catalyst_experiments}

We evaluate a practical case study: optimizing catalyst composition to maximize hydrogen yield. Let $f:\cX\to\mathbb{R}$ denote the achieved yield for a composition $x\in\cX$. This setting naturally supports preference-based feedback: precise yield measurements can be slow or costly, whereas pairwise judgments between candidates are often much easier to obtain. In many laboratory workflows, early-time indicators (e.g., initial rates or partial conversion) are sufficiently predictive to rank two catalysts, making preference labels a viable and efficient proxy.

We use the {Open Catalyst Experiments 2024 (OCx24)} dataset released by \emph{FAIR Chemistry} (Meta AI) \citep{abed2024open}, which reports hydrogen yields for $63$ catalyst compositions formed from silver (Ag), gold (Au), and zinc (Zn). Each composition is a triplet $(x_{\mathrm{Ag}},x_{\mathrm{Au}},x_{\mathrm{Zn}})\in[0,1]^3$ with $x_{\mathrm{Ag}}+x_{\mathrm{Au}}+x_{\mathrm{Zn}}=1$. We set the utility $f$ to the recorded hydrogen yield for each composition and synthetically generate preference feedback from $f$ in our implementation.

Figures~\ref{fig:cat_1} and~\ref{fig:cat_2} plot regret for PF-TS (ours), MR-LPF \citep{kayal2025bayesian}, MaxMinLCB \citep{pasztor2024bandits}, and POP-BO \citep{xu2024principled}, using the same inference procedure as Section~\ref{prediction_uncertainty} and a Mat\'ern-$5/2$ kernel with hyperparameters same as the previous experiments. For reference, we also include cumulative regret for vanilla GP-TS with direct scalar feedback on $f$. Results are averaged over 30 runs. Over $T=800$ iterations, PF-TS achieves the lowest instantaneous regret among BOHF methods and substantially lower cumulative regret than MR-LPF and POP-BO, while remaining competitive with MaxMinLCB.

\begin{figure}[t]
    \centering
    \includegraphics[width=0.9\linewidth]{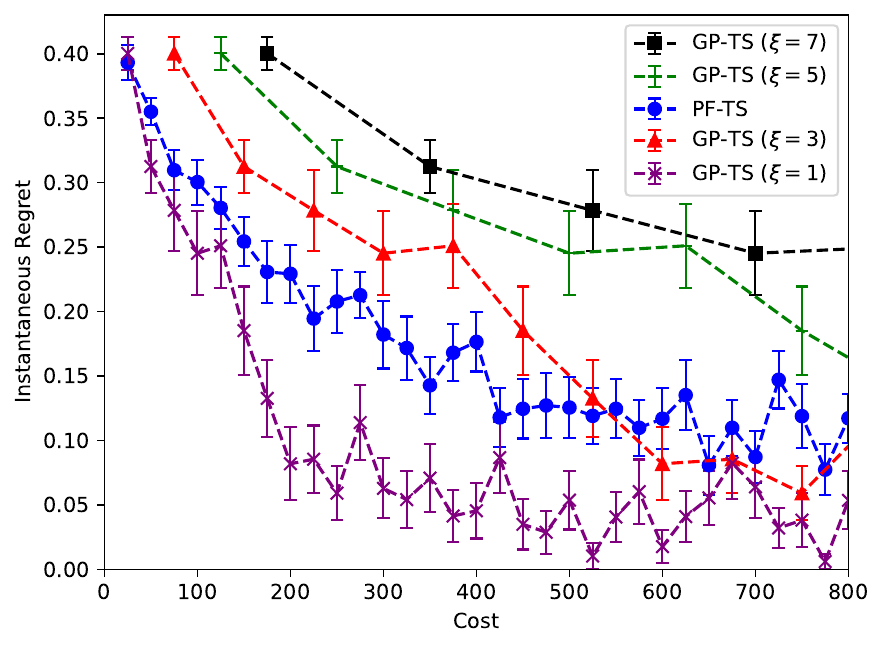}
    \caption{Instantaneous (simple) regret versus total cost $c$ (preference units) up to $800$ for PF-TS and GP-TS under cost ratios $\xi\in\{1,3,5,7\}$ with Catalyst data. For GP-TS, $c$ buys $c/\xi$ scalar iterations; PF-TS runs $c$ preference iterations. We plot regret and $\pm$1 standard error after each additional 25 iterations can be bought.}
    \label{fig:cat_cost}
    \vspace{-0.6cm}
\end{figure}
\subsection{Preference vs.\ Scalar Feedback}
Since the goal is to optimize the latent utility $f$, direct scalar observations are inherently more informative than pairwise preferences. Accordingly, in Figure~\ref{fig:cat_combined} the vanilla scalar-feedback method
(GP-TS)
accrues markedly less regret than BOHF methods when compared at the same \emph{number of iterations}. However, when preference labels are substantially cheaper to obtain than scalar measurements, BOHF can be advantageous under a fixed \emph{cost} budget.

To make this concrete, suppose a single scalar observation costs $\xi\!\ge\!1$ units relative to one preference query (cost $1$). In Figure~\ref{fig:cat_cost} we compare our PF-TS to GP-TS on the OCx24 dataset under a total cost budget $c$ (in preference units), for $\xi\in\{1,3,5,7\}$. PF-TS runs for $c$ preference iterations, while GP-TS is allotted only $c/\xi$ scalar iterations. We plot the instantaneous (simple) regret of each method at cost $c\le 800$. When preference feedback is sufficiently cheaper (larger $\xi$), PF-TS becomes preferable on a cost-adjusted basis.

%% file: CAMERA_READY/AISTATS_camera_ready/6_conclusion.tex
\section{CONCLUSION}
We introduced PF-TS, a Thompson-sampling approach to BOHF that leverages anchor independence and draws multiple posterior samples to select {both} actions each round in a fully sequential and {symmetric} manner. By our main guarantee (Regret bound in Theorem~\ref{theorem_regret}), PF-TS achieves cumulative regret $\tilde{\mathcal{O}}\!\big(\beta_T(\delta)\sqrt{T\,\Gamma(T)}\big)$, matching the order attained by GP-TS under scalar feedback. Empirically, on both synthetic (Ackley) and real (OCx24) objectives, PF-TS delivers competitive or improved performance to recent BOHF baselines at practical horizons.

Its symmetry in pair selection and full sequentiality make PF-TS a natural fit for applications where pairwise judgments are cheap and timely, e.g., human-in-the-loop design, LLM-as-judge evaluation, and high-throughput scientific discovery. Under cost budgets where scalar labels are more expensive than preferences, PF-TS can outperform standard BO on cost-adjusted regret. Looking ahead, promising directions include cost-aware hybrids of scalar and preference feedback, \emph{one-out-of-many} preference queries (in contrast to the \emph{one-out-of-two} BOHF setting), and large-scale experiments with LLM judges, including active data collection for fine-tuning large models.

\section*{Acknowledgments}

This work was completed while Joseph Lazzaro was an intern at MediaTek Research.

%% file: CAMERA_READY/AISTATS_camera_ready/checklist.tex
\section*{Checklist}

\begin{enumerate}

  \item For all models and algorithms presented, check if you include:
  \begin{enumerate}
    \item A clear description of the mathematical setting, assumptions, algorithm, and/or model. [Yes] This has been outlined in detail in Sections \ref{sec:prelim} and \ref{sec:ts_pref}.
    \item An analysis of the properties and complexity (time, space, sample size) of any algorithm. [Yes] One main focus of this paper is the analysis of our proposed PF-TS algorithm, with both theoretical results in Section \ref{sec:analysis} and empirical studies in Section \ref{sec:experiments}.
    \item (Optional) Anonymized source code, with specification of all dependencies, including external libraries. [Yes] This will be included in the supplementary material. 
  \end{enumerate}

  \item For any theoretical claim, check if you include:
  \begin{enumerate}
    \item Statements of the full set of assumptions of all theoretical results. [Yes] Assumptions have been clearly stated in Section \ref{sec:prelim} and for Theorem statements in Section \ref{sec:analysis}. Any missing details are readily available in the appendix.
    \item Complete proofs of all theoretical results. [Yes] Complete proofs can be found in the appendix.
    \item Clear explanations of any assumptions. [Yes] Clear explanations are given in Sections \ref{sec:prelim} and \ref{sec:analysis}.     
  \end{enumerate}

  \item For all figures and tables that present empirical results, check if you include:
  \begin{enumerate}
    \item The code, data, and instructions needed to reproduce the main experimental results (either in the supplemental material or as a URL). [Yes] Code and directions can be found in the supplementary material.
    \item All the training details (e.g., data splits, hyperparameters, how they were chosen). [Yes] This is discussed in detail in the appendix.
    \item A clear definition of the specific measure or statistics and error bars (e.g., with respect to the random seed after running experiments multiple times). [Yes] We state and include standard error bars in all plots.
    \item A description of the computing infrastructure used. (e.g., type of GPUs, internal cluster, or cloud provider). [Yes] This is discussed in the appendix.
  \end{enumerate}

  \item If you are using existing assets (e.g., code, data, models) or curating/releasing new assets, check if you include:
  \begin{enumerate}
    \item Citations of the creator If your work uses existing assets. [Yes]
    \item The license information of the assets, if applicable. [Not Applicable]
    \item New assets either in the supplemental material or as a URL, if applicable. [Not Applicable]
    \item Information about consent from data providers/curators. [Not Applicable]
    \item Discussion of sensible content if applicable, e.g., personally identifiable information or offensive content. [Not Applicable]
  \end{enumerate}

  \item If you used crowdsourcing or conducted research with human subjects, check if you include:
  \begin{enumerate}
    \item The full text of instructions given to participants and screenshots. [Not Applicable]
    \item Descriptions of potential participant risks, with links to Institutional Review Board (IRB) approvals if applicable. [Not Applicable]
    \item The estimated hourly wage paid to participants and the total amount spent on participant compensation. [Not Applicable]
  \end{enumerate}

\end{enumerate}

%% file: CAMERA_READY/AISTATS_camera_ready/appendix.tex
\appendix
\section*{Appendix}

\section{Notation Summary}

\begin{table}[h]
\centering
\small
\begin{tabular}{ll}
\hline
\textbf{Notation} & \textbf{Description} \\
\hline
$\mathcal{X}$ &
Action (design) space. \\

$x, x', x_t, x'_t$ &
Actions selected at a given round; $(x_t, x'_t)$ is the queried pair at round $t$. \\

$T$ &
Total number of optimization rounds (horizon). \\

$f : \mathcal{X} \to \mathbb{R}$ &
Unknown latent utility function to be maximized. \\

$x^\star$ &
Globally optimal action, $x^\star \in \arg\max_{x \in \mathcal{X}} f(x)$. \\

$y_t \in \{0,1\}$ &
Observed preference feedback at round $t$; $y_t = 1$ if $x_t \succ x'_t$. \\

$\mu(\cdot)$ &
Known monotone link function mapping utility differences to preference probabilities
(e.g., logistic link). \\

$h(x,x')$ &
Latent comparison function, defined as $h(x,x') = f(x) - f(x')$. \\

$k(\cdot,\cdot)$ &
Base positive-definite kernel on $\mathcal{X} \times \mathcal{X}$. \\

$\mathcal{H}_k$ &
Reproducing kernel Hilbert space (RKHS) induced by kernel $k$. \\

$\|f\|_{\mathcal{H}_k}$ &
RKHS norm of the latent utility function $f$. \\

$k^\Delta(\cdot,\cdot)$ &
Dueling (difference) kernel on $\mathcal{X} \times \mathcal{X}$ defined from $k$. \\

$\mathcal{H}_{k^\Delta}$ &
RKHS induced by the dueling kernel $k_\Delta$. \\

$h_t$ &
Regularized empirical risk minimizer estimating $h$ at round $t$. \\

$k^{\Delta}_t(\cdot, \cdot)$ &
Posterior covariance (uncertainty) kernel for $h$ after $t$ observations. \\

$\sigma_t^2(z)$ &
Posterior variance proxy at pair $z=(x,x')$, given by $k_{\Delta,t}(z,z)$. \\

$\lambda$ &
Regularization parameter in kernelized estimation. \\

$\kappa$ &
Uniform bound on the inverse derivative of the link function $\mu$. \\

$\Gamma(t)$ &
Maximum information gain after $t$ preference queries under kernel $k_\Delta$. \\

$\beta_t(\delta)$ &
Confidence-width multiplier ensuring high-probability error bounds. \\

$\tilde h_t$ &
Thompson sample from the Gaussian surrogate posterior over comparison functions. \\

$\tilde f_t$ &
Latent utility function corresponding to a separable Thompson sample $\tilde h_t$. \\

$v_t$ &
Exploration scale used in Thompson sampling (typically set to $\beta_t(\delta)$). \\

$R(T)$ &
Cumulative regret over $T$ rounds. \\

\hline
\end{tabular}
\caption{Summary of notation used throughout the main body of the paper.}
\end{table}

\section{Background Material} \label{appendix:background}

This section provides standard background material on Bayesian Optimisation, included for completeness and to aid in the interpretation of our results. 

\subsection{Reproducing Kernel Hilbert Space (RKHS)}
Let $\mathcal{X}$ be a nonempty set and let $k : \mathcal{X} \times \mathcal{X} \to \mathbb{R}$ be a
\emph{positive-definite kernel}, i.e.,
for any $n\in\mathbb{N}$, any points $x_1,\dots,x_n \in \mathcal{X}$ and any coefficients
$c_1,\dots,c_n \in \mathbb{R}$,
\[
    \sum_{i=1}^n\sum_{j=1}^n c_i c_j k(x_i, x_j) \ \geq \ 0.
\]
The RKHS $\mathcal{H}_k$ associated with $k$ is the unique Hilbert space of functions
$f:\mathcal{X}\rightarrow \mathbb{R}$ satisfying:
\begin{enumerate}
    \item For each $x \in \mathcal{X}$, the \emph{representer function} $k(\cdot, x)$ belongs to $\mathcal{H}_k$.
    \item \emph{Reproducing property:} For all $f \in \mathcal{H}_k$ and all $x \in \mathcal{X}$,
    \[
        f(x) \ =\ \langle f, \, k(\cdot, x) \rangle_{\mathcal{H}_k}.
    \]
\end{enumerate}
Intuitively, $\mathcal{H}_k$ is the completion of the linear span of $\{k(\cdot, x) : x \in \mathcal{X}\}$
with respect to the inner product induced by $k$. The RKHS norm $\|f\|_{\mathcal{H}_k}$ measures the smoothness of $f$
relative to the kernel $k$, and plays a central role in the theoretical analysis of BO.

\subsection{Mercer’s Theorem}

Mercer’s theorem provides a spectral decomposition of
continuous, symmetric, and positive-definite kernels on compact domains, and
connects the kernel $k$ with an orthonormal basis of the associated
$L^2$ space.

Let $k: \mathcal{X} \times \mathcal{X} \to \mathbb{R}$ be a continuous symmetric
positive-definite kernel, with $\mathcal{X} \subset \mathbb{R}^d$ compact and
equipped with a probability measure $\mu$. Mercer’s theorem states that there
exists a sequence $\{\gamma_m\}_{m=1}^\infty$ of positive eigenvalues and an
orthonormal set of eigenfunctions $\{\varphi_m\}_{m=1}^\infty$ in $L^2(\mathcal{X}, \mu)$
such that the kernel admits the expansion
\begin{equation}
    k(x, x') \ =\ \sum_{m=1}^{\infty} \gamma_m \, \varphi_m(x) \, \varphi_m(x'),
    \label{eq:mercer}
\end{equation}
where the series converges absolutely and uniformly on
$\mathcal{X} \times \mathcal{X}$.

For the RKHS $\mathcal{H}_k$ induced by $k$, the eigenfunctions
$\{\varphi_m\}$ form an orthonormal basis of $L^2(\mathcal{X},\mu)$, but not
necessarily of $\mathcal{H}_k$. Any function $f \in \mathcal{H}_k$ can be written
as
\begin{equation}
    f(x) \ =\ \sum_{m=1}^{\infty} w_m \sqrt{\gamma_m} \, \varphi_m(x),
    \label{eq:rkhs-expansion}
\end{equation}
for some coefficients $\{w_m\}_{m=1}^\infty$ with
\[
    \| f \|_{\mathcal{H}_k}^2 \ =\ \sum_{m=1}^{\infty} w_m^2.
\]
This expansion makes explicit the connection between the geometry of
$\mathcal{H}_k$ and the spectral properties of the kernel $k$.



\subsection{Common Kernel Functions}
The choice of kernel $k$ in a GP determines the properties of the functions drawn from it.
Two widely used kernels are described below.

\paragraph{Squared Exponential (RBF) Kernel.}
The squared exponential (SE), or radial basis function (RBF), kernel is given by
\[
    k_{\mathrm{SE}}(x,x') = \sigma_f^2
    \exp\left(-\frac{\|x - x'\|_2^2}{2\ell^2}\right),
\]
where $\ell > 0$ is the characteristic lengthscale and $\sigma_f^2$ is the signal variance.
This kernel is infinitely mean-square differentiable, encoding a prior belief in extreme smoothness.

\paragraph{Matérn Kernel.}
The Matérn kernel is a generalisation that allows control over the smoothness
of sample paths via a parameter $\nu>0$:
\[
    k_{\mathrm{Mat\acute{e}rn}}(x,x')
    = \sigma_f^2 \frac{2^{1-\nu}}{\Gamma(\nu)}
    \left( \frac{\sqrt{2\nu}\, r}{\ell} \right)^\nu
    K_{\nu}\left( \frac{\sqrt{2\nu}\, r}{\ell} \right),
\]
where $r = \|x - x'\|_2$, $\Gamma$ is the gamma function,
and $K_{\nu}$ is the modified Bessel function of the second kind.
When $\nu = 1/2$, the kernel reduces to the exponential kernel; when $\nu \to \infty$,
it converges to the squared exponential kernel. Finite $\nu$ yields sample paths that are
$\lceil \nu - 1 \rceil$ times mean-square differentiable.


\section{Proof of Proposition~\ref{prop_theta_concentration}}\label{appendix:anchor}

By construction (see~\eqref{eq:pred}), $h_t\in\cH_{\kd}$ is the minimizer of a regularized empirical risk; by the representer theorem there exist coefficients $\{\bar\theta_i\}_{i=1}^t$ such that
\begin{equation}\label{eq:repr-ht}
h_t(\cdot,\cdot)\;=\;\sum_{i=1}^t \bar\theta_i\,\kd\!\big((\cdot,\cdot),(x_i,x'_i)\big).
\end{equation}
Since $\kd$ is the \emph{dueling} kernel induced by a base kernel $k$ on $\cX$,
\[
\kd\!\big((x,x'),(u,u')\big)\;=\;k(x,u)+k(x',u')-k(x,u')-k(x',u),
\]
and therefore $h_t$ is a \emph{difference} function: there exists $g_t\in\cH_k$ with
\begin{equation}\label{eq:mean-diff}
h_t(x,x')\;=\;g_t(x)-g_t(x')\qquad\forall(x,x')\in\cX^2.
\end{equation}
(Equivalently, $T:\cH_k\to\cH_{\kd}$ defined by $(Tf)(x,x')=f(x)-f(x')$ is an isometry onto; \eqref{eq:repr-ht} lies in its range.)

Let $\tilde h_t\sim\gp(h_t,\kd_t)$ be the Thompson sample on $\cX^2$. We will construct $\tilde f_t$ explicitly and then verify the identity. Fix an arbitrary anchor $x_0\in\cX$ and define
\[
\tilde f_t(x)\;:=\;\tilde h_t(x,x_0),\qquad x\in\cX.
\]
Consider the (centered) discrepancy
\[
S(x,x')\;:=\;\tilde h_t(x,x')-\big(\tilde f_t(x)-\tilde f_t(x')\big)
\;=\;\tilde h_t(x,x')-\tilde h_t(x,x_0)+\tilde h_t(x',x_0).
\]
Since $\tilde h_t$ is Gaussian, $S(x,x')$ is Gaussian for each $(x,x')$.
We show that $S(x,x')$ is almost surely zero.

\emph{Mean of $S$.}
Using \eqref{eq:mean-diff},
\[
\mathbb{E}\,S(x,x')\;=\;h_t(x,x')-h_t(x,x_0)+h_t(x',x_0)
\;=\;\big(g_t(x)-g_t(x')\big)-\big(g_t(x)-g_t(x_0)\big)+\big(g_t(x')-g_t(x_0)\big)\;=\;0.
\]

\emph{Variance of $S$.}
We next show $\Var(S(x,x'))=0$. It suffices to verify that the covariance kernel $\kd_t$ inherits the dueling \emph{separable-difference} structure.
Write the usual kernel ridge–style posterior covariance (on any finite set) as
\[
\kd_t(\cdot,\cdot)\;=\;\kd(\cdot,\cdot)\;-\;\kd_{\!Z}(\cdot)^\top\big(\Kd_Z+\lambda\kappa I\big)^{-1}\kd_{\!Z}(\cdot),
\]
where $Z=\{(x_i,x'_i)\}_{i=1}^t$, $\Kd_Z=[\kd(z_i,z_j)]_{i,j=1}^t$ and $\kd_{\!Z}(z)=[\kd(z,z_i)]_{i=1}^t$. Each term on the right-hand side is a finite linear combination of \emph{dueling} kernels, hence $\kd_t$ is again dueling: there exists a positive semidefinite kernel $k_t$ on $\cX$ such that
\begin{equation}\label{eq:duel-form-post}
\kd_t\!\big((x,x'),(u,u')\big)\;=\;k_t(x,u)+k_t(x',u')-k_t(x,u')-k_t(x',u).
\end{equation}
(One concrete choice is $k_t(x,u):=\kd_t\!\big((x,x_0),(u,x_0)\big)$; positive semidefiniteness follows from that of $\kd_t$.)

Now expand the variance using~\eqref{eq:duel-form-post}:
\begin{align*}
\Var\big(S(x,x')\big)
&=\kd_t\!\big((x,x'),(x,x')\big)
 +\kd_t\!\big((x,x_0),(x,x_0)\big)
 +\kd_t\!\big((x',x_0),(x',x_0)\big) \\
&\quad -2\,\kd_t\!\big((x,x'),(x,x_0)\big)
 +2\,\kd_t\!\big((x,x'),(x',x_0)\big)
 -2\,\kd_t\!\big((x,x_0),(x',x_0)\big) \\
&=\big[k_t(x,x)+k_t(x',x')-2k_t(x,x')\big]
 +\big[k_t(x,x)+k_t(x_0,x_0)-2k_t(x,x_0)\big] \\
&\quad +\big[k_t(x',x')+k_t(x_0,x_0)-2k_t(x',x_0)\big] \\
&\quad -2\big[k_t(x,x)+k_t(x',x_0)-k_t(x,x_0)-k_t(x',x)\big] \\
&\quad +2\big[k_t(x,x')+k_t(x',x_0)-k_t(x,x_0)-k_t(x',x')\big] \\
&\quad -2\big[k_t(x,x')+k_t(x_0,x_0)-k_t(x,x_0)-k_t(x',x_0)\big] \\
&=0,
\end{align*}
by term-by-term cancellation. Hence $S(x,x')$ is Gaussian with zero mean and zero variance, so $S(x,x')=0$ almost surely.

\emph{Conclusion.}
For all $(x,x')\in\cX^2$,
\[
\tilde h_t(x,x')\;=\;\tilde f_t(x)-\tilde f_t(x')\qquad\text{a.s.,}
\]
with the explicit choice $\tilde f_t(x)=\tilde h_t(x,x_0)$ (unique up to an additive constant). This proves the proposition.

\medskip
\emph{Remarks on representer structure:}
The use of the representer theorem above ensures the \emph{mean} $h_t$ admits a decomposition through $k$ via \eqref{eq:mean-diff}. For the sampled component, one can also view finite-dimensional marginals of $\tilde h_t$ as living in the linear span of kernel sections $\{\kd(\cdot,z_i)\}_{i=1}^t$; because $\kd$ is separable through $k$, the difference property is preserved by posterior linear updates, yielding \eqref{eq:duel-form-post}. The anchor construction then provides the pathwise identification $\tilde f_t(x)=\tilde h_t(x,x_0)$ and completes the argument.

\section{Proof of Theorem \ref{theorem_regret}} \label{app:theorem_proof}



\paragraph{Structure of the Proof.}
The structure of the proof is as follows: In Sections~\ref{proof_concentration}--\ref{proof_events}, we develop high-probability concentration inequalities for posterior samples and define uniform events ensuring reliable uncertainty control. Section~\ref{subsec_anti_conc} introduces an anti-concentration argument that guarantees the selection of `unsaturated' actions with sufficient probability, thereby ensuring exploration. Finally, by combining these lemmas with information-gain bounds, the analysis (Sections~\ref{subsec_proof_least_var}--\ref{subsec_proof_putting_together}) derives the cumulative regret rate $\tilde{O}(\beta_T(\delta)\sqrt{T\Gamma(T)})$, matching that of standard GP-TS under scalar feedback. Before beginning the proof, we briefly comment on continuous action spaces.

\paragraph{Continuous action spaces} We emphasize that the statement in Theorem \ref{theorem_regret} is valid for both discrete \emph{and} continuous action spaces. For technical reasons, in the case of continuous action spaces we need to follow existing work analyzing TS (for vanilla BO), by discretizing the action space in each round. In particular, we will pick a decision set $\mathcal{C}_t \subset \mathcal{X}$. Namely, the actions in PF-TS will be selected as 
\[
x_t \in \argmax_{x\in\mathcal{C}_t}\ \tilde h_t^{(1)}(x,x_0),
\qquad
x'_t \in \argmax_{x\in\mathcal{C}_t}\ \tilde h_t^{(2)}(x,x_0).
\]
For the remainder of the proof we will consider this latter setting of a continuous action space. However, we note that the same order regret is achievable for finite, constant size action spaces with only minor modifications of the proof. We will choose the same discretization as \citep{chowdhury2017kernelized}. Namely, without loss of generality, let $\mathcal{X} = [0,w]^d \subset \mathbb{R}^d$. Then let the decision set $\mathcal{C}_t$ be evenly spaced across the action space with  $$|\mathcal{C}_t| = (BGwdt^2)^d,$$ where $B$ is the norm bound from Assumption \ref{ass:RKHS_norm}, $w$ is the width of the action space in each dimension, $d$ is the number of dimensions, and we define $$G = \sup_{x \in \mathcal{X}} \sup_{j \in [d]} \left(\frac{\partial^2 k(p,q)}{\partial p_j \partial q_j}\big|_{p=q=x}\right).$$ 
 This choice helps us control the error caused by the discretisation. In particular, let $[x]_t$ be the closest element of the decision set to $x \in \mathcal{X}$. Then, as discussed in \citep[Section 4.2, ][]{chowdhury2017kernelized}, selecting a decision set of this size guarantees for all $x\in \mathcal{X}$ that 
\begin{equation}\label{eqn:discretisation_fine}
    |f(x)-f([x]_t)|< 1/t^2.
\end{equation}
This will ensure that the decision set is sufficiently fine as to not accumulate excessive regret, as described blow. Finally, regardless of the action space we define the exploration scale $$v_t = \beta_t(\delta).$$



\paragraph{Discrete regret \& Lipschitz reduction} 

We begin by reducing the regret expressed in terms of preference probabilities to a form involving utility differences, which are easier to analyze. Recall that under the BTL model the probability of preference is given by $\Pr(x\succ x')=\mu\big(h(x,x')\big)$ with an increasing $L_\mu$–Lipschitz link. Because the regret definition depends on these probabilities, directly controlling it is inconvenient. However, the Lipschitz property of the link function allows us to relate differences in preference probabilities to differences in the latent utility function. This enables us to bound the instantaneous dueling regret in terms of utility gaps.

In addition, following standard analyses of Thompson sampling for continuous action spaces, we work with a discretized decision set and introduce a notion of discrete regret with respect to the best action in this set. We denote the instantaneous dueling regret with respect to the global optimum by $r_t$ and the discrete regret with respect to the best action in the decision set by $\Delta_t$. Using the Lipschitz property of the link function, we obtain the following relationship.
\begin{align}
    r_t &:=\frac{\mu\big(h(x^\star,x_t)\big)+\mu\big(h(x^\star,x'_t)\big)-1}{2}
\;\le\; \frac{L_\mu}{2}\,\Big(h(x^\star,x_t)+h(x^\star,x'_t)\Big),\nonumber\\
\Delta_t &:=\frac{\mu\big(h([x^\star]_t,x_t)\big)+\mu\big(h([x^\star]_t,x'_t)\big)-1}{2}
\;\le\; \frac{L_\mu}{2}\,\Big(h([x^\star]_t,x_t)+h([x^\star]_t,x'_t)\Big).\label{eqn_discrete_regret_lipschitz}
\end{align}
Furthermore, using the Lipschitz property again and as a consequence of the fineness of the decision set \eqref{eqn:discretisation_fine} we have 
\begin{align}
    |r_t - \Delta_t| &\leq \frac{L_\mu}{2}|h(x^*, x_t) + h(x^*, x_t') - h([x^*]_t, x_t) + h([x^*]_t, x_t')|\nonumber\\
    &\leq L_\mu |f(x^*) - f([x^*]_t)|\nonumber\\
    &\leq \frac{L_\mu }{t^2} \label{eqn_discrete_cont_regret_distance}
\end{align}

\subsection{Concentration of a posterior sample}\label{proof_concentration}
We have the following general concentration result for a posterior sample of a difference function, which will be useful for our analysis. Intuitively, it shows that a Thompson sample remains close to the posterior mean, with deviations controlled by the posterior standard deviation. This allows us to relate sampled comparison functions to the estimated function in a uniform way over the candidate set. Here we denote the filtration $\mathcal{F}t$ as the history up until time $t$, namely $H{t-1}$ defined in Section \ref{sec:prelim}.

\begin{lemma} \label{lem:posterior_concentration}
    Consider a posterior sample of the difference function $\tilde{h}_t\sim GP(h_t, v_t^2k_t^\Delta)$. Let $\mathcal{C}$ be a finite subset of the action space $\mathcal{X}$ and $\delta \in (0,1)$, then 
    \begin{equation*}
        \mathbb{P}\left(\forall x,x'\in \mathcal{C}, |\tilde{h}_t(x,x') - h_t(x,x')|\leq v_t\sqrt{2\log\left(|\mathcal{C}|^2 /\delta \right)} \sigma_{t}(x,x') \bigg| \mathcal{F}_t\right) \geq 1-\delta
    \end{equation*}
\end{lemma}
\begin{proof}
    As discussed in Appendix \ref{appendix:background}, we know that for any $x,x' \in D'$ that the posterior sample at that pair has distribution $\tilde{h}_t(x,x') \sim N(h_t(x,x'), v_t^2\sigma^2_{t}(x,x'))$. Hence, from \citep[Lemma B4, ][]{hoffman2013exploiting} we know that for any $x,x' \in D'$.
    \begin{equation*}
        \mathbb{P}\left(|\tilde{h}_t - h_t(x,x')| \leq v_t\sqrt{2\log\left(1 /\delta \right)} \sigma_{t}(x,x')\bigg| \mathcal{F}_t\right) \geq 1-\delta.
    \end{equation*}
    Hence, by taking a union bound over all $|\mathcal{C}|^2$ pairs of actions from $\mathcal{C}$, we attain the desired result.
\end{proof}

\subsection{Uniform high probability events}\label{proof_events}

 We now introduce two high probability events which uniformly control the distance between the estimated difference function $h_t$ and the true difference function $h$, as well as the difference between the estimated difference function $h_t$ and each of the sampled difference functions $\tilde{h}_t^{(i)}$. Namely, define the following events, each required to hold \emph{for all} $z\in\cX^2$:
\[
E_1:\ \ |h(z)-h_t(z)| \le \beta_t(\delta)\,\sigma_t(z)\ \  \text{for all}\ \  t\le T,\quad
E_{2,t}:\ \ \big|\tilde h_t^{(i)}(z)-h_t(z)\big| \le \beta_t(\delta)\sqrt{2\log\left(|\mathcal{C}_t|^2 t^2 \right)}\,\sigma_t(z)\ \ \text{for }i\in\{1,2\}.
\]
By Lemma \ref{lem:confidence_bound} and Lemma \ref{lem:posterior_concentration}, we know that these events hold with high probability, namely
\begin{align}
    \mathbb{P}\left(E_1\right) &\geq 1-\delta \label{eqn:E1_prob}\\
    \mathbb{P}\left(E_{2,t}|\mathcal{F}_t\right) &\geq 1-\frac{2}{t^2}\label{eqn:E2_prob}
\end{align}

Furthermore, for ease of notation, we define the constant 
\begin{equation}\label{eqn_d_t_definition}
    d_t = \beta_t(\delta)\left(1+\sqrt{2\log\left(|\mathcal{C}_t|^2 t^2 \right)}\right)
\end{equation}
and observe that, under both events $E_1 \cap E_{2,t}$ we have for all $t\le T$, all $z\in\cX^2$, and for $i\in\{1,2\}$
\begin{equation}\label{eqn:posterior_is_d_sigma_close}
    \big|\tilde h_t^{(i)}(z)-h(z)\big| \leq d_t \sigma_t(z).
\end{equation}

\subsection{Anti-concentration ensures unsaturated selection}\label{subsec_anti_conc}

To analyze the exploration behavior of the algorithm, we partition the action space into saturated and unsaturated points relative to a chosen anchor. Intuitively, a point is saturated if its posterior uncertainty is small compared to its sub-optimality gap, meaning the algorithm is already confident it is suboptimal. Unsaturated points are those for which uncertainty remains large enough that they may still plausibly be optimal. More formally, we define the saturated set $S_t(a)$ as the set of actions whose posterior uncertainty with an anchor point $a$ is smaller than their discrete sub-optimality gap. Namely 
\begin{equation}\label{eqn:def_saturated_set}
    S_t(a)\;:=\;\big\{x\in \mathcal{C}_t:\ h([x^\star]_t,x)\ >\ d_t\sigma_t(x,a)\big\}.
\end{equation}
We will refer to actions as unsaturated (with respect to $a$) when they are not in the saturated set $S_t(a)$.\\

We can then extend the Gaussian anti-concentration arguments used in \citep{chowdhury2017kernelized} in order to prove that we will always play an unsaturated action with some probability, through the following two lemmas. Intuitively, saturated points are those whose sub-optimality gap already exceeds their posterior uncertainty, meaning the algorithm has enough evidence that they are suboptimal. In contrast, unsaturated points still have sufficiently large uncertainty relative to their gap and therefore require further exploration. The following results show that the randomness of the Thompson sample ensures that such points are selected with non-negligible probability: Lemma \ref{lemma_anti_concentration} establishes a constant-probability anti-concentration property of the posterior sample, and Lemma \ref{lemma_high_prob_unsaturated} uses this to show that the maximizer of the sampled function lies outside the saturated set with constant probability. This guarantees that the algorithm continues to explore points whose uncertainty remains significant, which will be a key ingredient for controlling the regret in the subsequent analysis.\\

\begin{lemma}\label{lemma_anti_concentration}
    For any filtration $\mathcal{F}_t$ in which $E_1$ is true, for any posterior sample $\tilde{h}_t$ and for any $x,x' \in \mathcal{X}$ we can lower bound the probability
$$ \mathbb{P}\left(\tilde{h}_t(x,x')>h(x,x')\bigg| E_1\right) \geq p$$
for some universal constant $p = \frac{1}{4e\sqrt{\pi}}$.
\end{lemma}

\begin{proof}
    We first rewrite the event in terms of a standard normal random variable, recalling that a posterior sample for the difference function at $x,x' \in \mathcal{X}$ is distributed $\tilde{h}_t(x,x') \sim N(h_t(x,x'),v_t^2\sigma_t^2(x,x'))$. Let $\theta_t = \frac{|h(x,x') - h_t(x,x')|}{v_t\sigma_t(x,x')}$, then we can write
    \begin{align*}
        \mathbb{P}\left(\tilde{h}_t(x,x')>h(x,x')\bigg| E_1\right) &\geq \mathbb{P}\left(\frac{\tilde{h}_t(x,x') - h_t(x,x')}{v_t \sigma_t(x,x')}>\frac{h(x,x') - h_t(x,x')}{v_t\sigma_t(x,x')}\bigg| E_1\right)\\
        &\geq \mathbb{P}\left(\frac{\tilde{h}_t(x,x') - h_t(x,x')}{v_t \sigma_t(x,x')}>\frac{|h(x,x') - h_t(x,x')|}{v_t\sigma_t(x,x')}\bigg| E_1\right)\\
        &\geq \frac{1}{4\theta_t\sqrt{\pi}} \exp(-\theta_t^2)\\
        &\geq \frac{1}{4e\sqrt{\pi}}
    \end{align*}
    The penultimate inequality holds from via a Gaussian anti-concentration \citep[Lemma 7, ][]{chowdhury2017kernelized}. The final inequality holds since, by definition under event $E_1$ (see Section \ref{proof_events}), 
    $$\theta_t = \frac{|h(x,x') - h_t(x,x')|}{v_t\sigma_t(x,x')} \leq 1$$
\end{proof}

\begin{lemma}\label{lemma_high_prob_unsaturated}
For any filtration $\mathcal{F}_{t}$ such that $E_1$ is true and for any $a,a' \in \mathcal{C}_t$,

\begin{align*}
    \mathbb{P}\left[ x_t \in \mathcal{C}_t \setminus S_t(a) \mid \mathcal{F}_{t} \right] &\ge p - \frac{2}{t^2}\\
    \mathbb{P}\left[ x_t' \in \mathcal{C}_t \setminus S_t(a') \mid \mathcal{F}_{t} \right] &\ge p - \frac{2}{t^2},\\
\end{align*}
with $p = \frac{1}{4e\sqrt{\pi}}.$

\end{lemma}

\begin{proof}
We prove the first statement. The second statement is true by a symmetric argument. By anchor independence (see Section \ref{subsection:anchor_independence}), we know that the first action $x_t$ is the maximizer of the single argument search

\begin{equation*}
    x_t = \argmax_{x \in \mathcal{C}_t} \tilde{h}_t^{(1)}(x,a)
\end{equation*}
for any $a \in \mathcal{C}_t$, regardless of what anchor point is used in the implementation of the algorithm $x_0$.\\ 

If $\tilde{h}_t^{(1)}([x^\star]_t,a)$ is greater than $\tilde{h}_t^{(1)}(x,a)$ for all saturated points at round $t$ (namely, $\tilde{h}_t^{(1)}([x^\star]_t,a) > \tilde{h}_t^{(1)}(x,a), \forall x \in S_t(a)$), then one of the unsaturated points in $\mathcal{C}_t$ (which always includes $[x^\star]_t$) must be played and hence $x_t \in \mathcal{C}_t \setminus S_t(a)$. This implies
\begin{equation}\label{eqn_proof_quick}
\mathbb{P}\left[ x_t \in \mathcal{C}_t \setminus S_t(a) \mid \mathcal{F}_{t} \right] \ge 
\mathbb{P}\left[ \tilde{h}_t^{(1)}([x^\star]_t,a) > \tilde{h}_t^{(1)}(x,a), \forall x \in S_t \mid \mathcal{F}_{t} \right].
\end{equation}

Now, by the definition of the saturated set $S_t(a)$, we know $h([x^*]_t,x) > d_t \sigma_{t-1}(x,a)$, for all $x \in S_t(a)$. Furthermore, if both the events $E_1$ and $E_{2,t}$ are true, then $\tilde{h}_t(x,a) \le h(x,a) + d_t \sigma_{t-1}(x,a)$, for all $x \in \mathcal{C}_t$. Thus, if events $E_1$ and $E_{2,t}$ hold and using the separability of $h$, we have for all $x \in S_t(a)$, $\tilde{h}_t(x,a) < h(x,a) + h([x^*]_t,x)= h([x^\star]_t,a)$. Therefore, for any filtration $\mathcal{F}_{t}$ such that $E_1$ is true, either $E_{2,t}$ is false, or else for all $x \in S_t(a)$, $\tilde{h}_t^{(1)}(x,a) < h([x^\star]_t,a)$. Hence, for any $\mathcal{F}_{t}$ such that $E_1$ is true,

\begin{align*}
    &\mathbb{P}\left[ \tilde{h}_t^{(1)}([x^\star]_t,a) > \tilde{h}_t^{(1)}(x,a), \forall x \in S_t(a) \mid \mathcal{F}_{t} \right]\\
     & \quad\quad \geq \mathbb{P}\left[ \tilde{h}_t^{(1)}([x^\star]_t,a) > h([x^\star]_t,a)\mid \mathcal{F}_{t} \right] - \mathbb{P}\left[E_{2,t}^C\mid \mathcal{F}_{t} \right]\\
     & \quad\quad \geq p - \frac{2}{t^2}
\end{align*}
For the final inequality we have used Lemma \ref{lemma_anti_concentration} and \eqref{eqn:E2_prob}. Now the proof follows by combining this final inequality with Equation \eqref{eqn_proof_quick}. 
\end{proof}

\subsection{Control via least-variance unsaturated points}\label{subsec_proof_least_var}

Having established that unsaturated points are selected with constant probability, we next relate the uncertainty at the played action to the smallest uncertainty among unsaturated actions. For this purpose, we introduce the notion of a least-variance unsaturated point, which is the unsaturated action whose posterior variance with respect to the anchor is minimal. Intuitively, this point represents the most “well-understood’’ action among those that still require exploration. The following argument shows that, because the algorithm selects unsaturated actions with constant probability, the expected posterior variance of the chosen action is lower bounded by a constant fraction of the variance of the least-variance unsaturated point. This connection allows us to control the uncertainty of the played actions in terms of the smallest remaining uncertainty among exploratory candidates, which will be crucial when bounding the regret in the subsequent steps of the proof.

Formally, for any potential anchor point $a\in\mathcal{C}_t$, define a least-variance unsaturated point as
\[
\bar x_t(a)\ \in\ \argmin_{x\in\cX\setminus S_t(a)}\ \sigma_t(x,a).
\]
Note that such a point always exists because we have $x^* \in \mathcal{C}_t\backslash S_t(a)$ for any $a\in\mathcal{C}_t$.

\begin{lemma} \label{lemma_unsaturated_ub}
Condition on any filtration $\cF_{t-1}$ such that $E_1$ holds. Let $a_t:=x'_t$ and $a'_t:=x_t$.
Then there exists an absolute constant $p\in(0,1)$ (the one from Lemma~\ref{lemma_high_prob_unsaturated}) such that
\begin{align}
\EE\!\big[\sigma_t(\bar x_t(a_t),a_t)\ \bigm|\ \cF_{t}\big]
&\ \le\ \frac{1}{p-\frac{2}{t^2}}\,\EE\!\big[\sigma_t(x_t,a_t)\ \bigm|\ \cF_{t}\big],\label{eq:witness_to_play_1}\\
\EE\!\big[\sigma_t(\bar x_t(a'_t),a'_t)\ \bigm|\ \cF_{t},\big]
&\ \le\ \frac{1}{p-\frac{2}{t^2}}\,\EE\!\big[\sigma_t(x'_t,a'_t)\ \bigm|\ \cF_{t}\big].\label{eq:witness_to_play_2}
\end{align}
\end{lemma}
\begin{proof} 
    We prove \eqref{eq:witness_to_play_1}; \eqref{eq:witness_to_play_2} is identical by symmetry.  
Note that, on the event
\[
U_t(a_t)\ :=\ \{\,x_t\notin S_t(a_t)\,\},
\]
we have by definition of $\bar x_t(a_t)$ that $\sigma_t(x_t,a_t)\ \ge\ \sigma_t(\bar x_t(a_t),a_t)$. 

Taking conditional expectations gives us
\begin{align*}
    \mathbb{E}\left[\sigma_t(x_t,a_t)\big|\cF_{t},\right] &\geq \mathbb{E}\left[\sigma_t(x_t,a_t) \mathds{1}\{U_t(a_t)\}\big|\cF_{t}\right]\\
    &= \sum_{a\in \mathcal{C}_t} \mathbb{E}\left[\sigma_t(x_t,a_t) \mathds{1}\{U_t(a_t)\}\big|\cF_{t},a_t=a\right] \Pr \left(a_t=a \big|\cF_{t}\right)\\
    &= \sum_{a\in \mathcal{C}_t} \mathbb{E}\left[\sigma_t(x_t,a) \mathds{1}\{U_t(a)\}\big|\cF_{t},a_t=a\right] \Pr \left(a_t=a \big|\cF_{t}\right)\\
    &= \sum_{a\in \mathcal{C}_t} \mathbb{E}\left[\sigma_t(\Bar{x}_t(a),a) \mathds{1}\{U_t(a)\}\big|\cF_{t},a_t=a\right] \Pr \left(a_t=a \big|\cF_{t}\right)\\
    &= \sum_{a\in \mathcal{C}_t}\sigma_t(\Bar{x}_t(a),a) \Pr\left[ U_t(a)\big|\cF_{t},a_t=a\right] \Pr \left(a_t=a \big|\cF_{t}\right)\\
    &\geq \left(p-\frac{2}{t^2}\right)\sum_{a\in \mathcal{C}_t}\sigma_t(\Bar{x}_t(a),a) \Pr \left(a_t=a \big|\cF_{t}\right)\\
    &= \left(p-\frac{2}{t^2}\right) \mathbb{E}\left[\sigma_t(\Bar{x}_t(a_t),a_t)\big|\cF_{t}\right]
\end{align*}
Where in the second line we have used the law of total expectation, the fourth line comes from the definition of $\bar x_t(a_t)$ as noted at the beginning of the prof, in the fifth line we can take $\sigma_t(\Bar{x}_t(a),a)$ since it is deterministic given the conditions of the expectation, and in the penultimate line we have used Lemma~\ref{lemma_high_prob_unsaturated}. Rearranging yields \eqref{eq:witness_to_play_1}. 
\end{proof}

\subsection{Using anchor independence}
We now present a lemma that exploits the anchor-independence property of PF-TS. In particular, Proposition~\ref{prop_theta_concentration} implies that the actions selected by the algorithm are invariant to the choice of anchor. Consequently, for the purpose of the analysis we are free to select anchors that simplify the regret decomposition. In particular, when bounding the instantaneous regret in round $t$, we choose the anchor points $a_t = x'_t$ for the first action and $a'_t = x_t$ for the second. This choice allows us to relate the regret directly to the posterior uncertainty of the played pair.

\begin{lemma}\label{lemma_instant_ub_high_prob}
    In round $t$ of the PF-TS algorithm, when both good events hold $E_1\cap E_{2,t}$ we have that the instantaneous regret can be upper bounded by 
    \begin{equation*}
        r_t \leq \frac{L_\mu d_t}{4} \left(2\sigma_t(\Bar{x}_t,x_t') + 2\sigma_t(\Bar{x}_t',x_t) +2\sigma_t(x_t,x_t')\right) +\frac{L_\mu}{t^2} 
    \end{equation*}
\end{lemma}
\begin{proof}
Firstly, we can upper bound the instantaneous regret with the discrete instantaneous regret with some error, as shown in equation \eqref{eqn_discrete_cont_regret_distance},
\begin{equation*}
    r_t \leq \Delta_t + \frac{L_\mu}{t^2}.
\end{equation*}
For the rest of the proof we focus on upper bounding $\Delta_t$. We start by defining the anchor points $a_t = x_t'$ and $a_t' = x_t$ adding and subtracting the utility of the two least-variance unsaturated points with respect to the two anchor points, namely $\Bar{x}_t = \Bar{x}_t(a_t)$ and $\Bar{x}_t' = \Bar{x}_t(a_t')$. Recall from the Lipschitz reduction in equation \eqref{eqn_discrete_regret_lipschitz} that 
\begin{align*}
    \frac{2\Delta_t}{L_\mu} &\leq f([x^*]_t) - \frac{f(x_t)+f(x'_t)}{2}\\
    &= (f([x^*]_t) - \frac{f(\Bar{x}_t) + f(\Bar{x}_t')}{2} + \frac{f(\Bar{x}_t) + f(\Bar{x}_t')}{2} - \frac{f(x_t)+f(x'_t)}{2}\\
    &= \frac{h([x^*]_t,\Bar{x}_t) + h([x^*]_t,\Bar{x}_t')}{2} + \frac{f(\Bar{x}_t) + f(\Bar{x}_t')}{2} - \frac{f(x_t)+f(x'_t)}{2}\\
    \intertext{Since we know that both $\Bar{x}_t$ and $\Bar{x}_t'$ are unsaturated with respect to $S_t(a_t)$ and $S_t(a_t')$, by definition \eqref{eqn:def_saturated_set} we have that}
    &\leq \frac{d_t}{2} \left(\sigma_t(\Bar{x}_t,a_t) + \sigma_t(\Bar{x}_t',a_t')\right) + \frac{f(\Bar{x}_t) + f(\Bar{x}_t')}{2} - \frac{f(x_t)+f(x'_t)}{2}\\
    \intertext{Now, let's add and subtract both $f(a_t)$ and $f(a_t')$}
    &= \frac{d_t}{2} \left(\sigma_t(\Bar{x}_t,a_t) + \sigma_t(\Bar{x}_t',a_t')\right) + \frac{f(\Bar{x}_t) - f(a_t) + f(\Bar{x}_t') - f(a_t')}{2} - \frac{f(x_t) - f(a_t) +f(x'_t) - f(a_t')}{2}\\
    &= \frac{d_t}{2} \left(\sigma_t(\Bar{x}_t,a_t) + \sigma_t(\Bar{x}_t',a_t')\right) + \frac{h(\Bar{x}_t,a_t) + h(\Bar{x}_t',a_t')}{2} - \frac{h(x_t,a_t) +h(x'_t,a_t')}{2}\\
    \intertext{Now, we know from \eqref{eqn:posterior_is_d_sigma_close} that under event $E_1\cap E_{2,t}$, our posterior sample is $d_t\sigma_t$ close to the true difference function. Hence,}
    &\leq \frac{d_t}{2} \left(\sigma_t(\Bar{x}_t,a_t) + \sigma_t(\Bar{x}_t',a_t')\right) \\
    &\qquad+ \frac{\tilde{h}^{(1)}(\Bar{x}_t,a_t) + d_t\sigma_t(\Bar{x}_t,a_t)  + \tilde{h}^{(2)}(\Bar{x}_t',a_t') + d_t\sigma_t(\Bar{x}_t',a_t') }{2}\\
    &\qquad- \frac{\tilde{h}^{(1)}(x_t,a_t) - d_t\sigma_t(x_t,a_t)  +\tilde{h}^{(2)}(x'_t,a_t') - d_t\sigma_t(x_t',a_t') }{2}\\
\intertext{
Now we use the selection rule of PF-TS together with the anchor-independence property (Proposition~\ref{prop_theta_concentration}). 
By anchor independence, the actions chosen by the algorithm are invariant to the anchor used in the maximization. 
Hence, although the algorithm is implemented with a fixed anchor $x_0$, we may equivalently view $x_t$ and $x'_t$ as the maximizers of the sampled functions with respect to the arbitrary anchors $a_t$ and $a_t'$, respectively. 
In particular, $x_t = \arg\max_x \tilde h^{(1)}(x,a_t)$ and $x'_t = \arg\max_x \tilde h^{(2)}(x,a_t')$. 
Therefore the sampled values at the selected actions dominate those at any other candidate points, including $\Bar{x}_t$ and $\Bar{x}_t'$. 
Substituting these inequalities into the previous expression allows the sampled-function terms to be upper bounded by their corresponding uncertainty terms, yielding}
&\leq \frac{d_t}{2} \left(\sigma_t(\Bar{x}_t,a_t) + \sigma_t(\Bar{x}_t',a_t') + \sigma_t(\Bar{x}_t,a_t) + \sigma_t(\Bar{x}_t',a_t') +  \sigma_t(x_t,a_t)  +\sigma_t(x_t',a_t')\right).
\end{align*}

Now, by picking $a_t = x_t'$ and $a_t'=x_t$, we can write this as 

\begin{align*}
    \frac{2\Delta_t}{L_\mu} &\leq \frac{d_t}{2} \left(\sigma_t(\Bar{x}_t,x_t') + \sigma_t(\Bar{x}_t',x_t) +\sigma_t(\Bar{x}_t,x_t') + \sigma_t(\Bar{x}_t',x_t) +  \sigma_t(x_t,x_t')  +\sigma_t(x_t',x_t)\right) \\
    &= \frac{d_t}{2} \left(2\sigma_t(\Bar{x}_t,x_t') + 2\sigma_t(\Bar{x}_t',x_t) +  \sigma_t(x_t,x_t')  +\sigma_t(x_t',x_t)\right) \\
    &= \frac{d_t}{2} \left(2\sigma_t(\Bar{x}_t,x_t') + 2\sigma_t(\Bar{x}_t',x_t) +2\sigma_t(x_t,x_t')\right) \\
\end{align*}
Where the final line holds by symmetry of $\sigma_t$. The proof is then complete.
\end{proof}

\subsection{Putting everything together}\label{subsec_proof_putting_together}

We now combine the results from the previous subsections to obtain the final regret bound. 
Lemma~\ref{lemma_instant_ub_high_prob} shows that the instantaneous regret can be controlled by posterior uncertainty terms involving the played pair $(x_t,x_t')$ and the least-variance unsaturated points. 
The role of Lemma~\ref{lemma_unsaturated_ub} is to relate the uncertainty of these unsaturated points to the uncertainty of the played pair: since the algorithm selects unsaturated actions with constant 
probability, the expected posterior variance of the chosen actions cannot be much smaller than that of the least-variance unsaturated points. Consequently, the first two terms in the bound of 
Lemma~\ref{lemma_instant_ub_high_prob} can be controlled in expectation by $\sigma_t(x_t,x_t')$.

This reduction is important because the cumulative sum of posterior uncertainties 
$\sigma_t(x_t,x_t')$ can be bounded using standard information-theoretic arguments. 
In particular, the total uncertainty accumulated by the queried pairs is controlled by 
the maximum information gain under the dueling kernel. The following lemma formalizes 
this relationship and allows us to convert the variance-based regret bound into a 
cumulative regret bound in terms of the information gain.

\begin{lemma}\label{lemma_expectation_instant_ub}
    We can update the conditional expectation of the instantaneous regret as follows, for any filtration $\mathcal{F}_t$ where $E_1$ holds. 
    \begin{equation*}
        \mathbb{E}\left[r_t | \cF_{t-1} \right]\leq C'L_\mu d_t\mathbb{E}\left[\sigma_t(x_t,x_t')| \cF_{t-1} \right] +\frac{L_\mu+1}{t^2},
    \end{equation*}
    where $C'$ is a universal constant.
\end{lemma}
\begin{proof} 

By taking the conditional expectation of the instantaneous regret and multiplying by $1=(\mathds{1}\{E_{2,t}\} + \mathds{1}\{E_{2,t}^C\})$, we get

\begin{align*}
    \mathbb{E}\left[r_t | \cF_{t-1} \right] &\leq \mathbb{E}\left[r_t \mathds{1}\{E_{2,t}\}| \cF_{t-1} \right] + \mathbb{E}\left[r_t \mathds{1}\{E_{2,t}^C\}| \cF_{t-1} \right]\\
    &\leq \mathbb{E}\left[\frac{L_\mu d_t}{4} \left(2\sigma_t(\Bar{x}_t,x_t') + 2\sigma_t(\Bar{x}_t',x_t) +2\sigma_t(x_t,x_t')\right) +\frac{L_\mu}{t^2}\right]+ \mathbb{E}\left[r_t \mathds{1}\{E_{2,t}^C\}| \cF_{t-1} \right]\\
    &\leq \left(\frac{L_\mu d_t}{p-\frac{1}{t^2}} + \frac{L_\mu d_t}{2}\right)\mathbb{E}\left[\sigma_t(x_t,x_t')| \cF_{t-1} \right] +\frac{L_\mu}{t^2} + \mathbb{E}\left[r_t \mathds{1}\{E_{2,t}^C\}| \cF_{t-1} \right]\\
    &\leq \left(\frac{L_\mu d_t}{p-\frac{1}{t^2}} + \frac{L_\mu d_t}{2}\right)\mathbb{E}\left[\sigma_t(x_t,x_t')| \cF_{t-1} \right] +\frac{L_\mu}{t^2} + \frac{1}{2}\Pr\left[\{E_{2,t}^C\}| \cF_{t-1} \right]\\
    &\leq \left(\frac{L_\mu d_t}{p-\frac{1}{t^2}} + \frac{L_\mu d_t}{2}\right)\mathbb{E}\left[\sigma_t(x_t,x_t')| \cF_{t-1} \right] +\frac{L_\mu}{t^2} + \frac{1}{t^2}\\
    &\leq C'L_\mu d_t\mathbb{E}\left[\sigma_t(x_t,x_t')| \cF_{t-1} \right] +\frac{L_\mu}{t^2} + \frac{1}{t^2}
\end{align*}
Where the second inequality uses Lemma \ref{lemma_instant_ub_high_prob}, the third inequality uses Lemma \ref{lemma_unsaturated_ub} and the symmetry of $\sigma_t$, the fourth inequality holds since $r_t \leq 1/2$, the penultimate inequality uses equation \eqref{eqn:E2_prob}, and the final inequality holds with $C' = 85$.
\end{proof}

For the remainder of the proof we use the upper bound on the instantaneous regret established in 
Lemma~\ref{lemma_expectation_instant_ub} to obtain a high-probability bound for the cumulative regret 
$R(T)$. Intuitively, we first control the expected instantaneous regret in terms of posterior 
uncertainty, and then convert this expected bound into a high-probability statement for the cumulative 
regret using a martingale concentration argument.

To do so, we decompose the cumulative regret into its conditional expectation and a martingale 
difference sequence capturing the stochastic fluctuations around this expectation. Lemma~\ref{lemma_martingale_true} 
formalizes this decomposition and identifies the corresponding martingale difference sequence. 
Lemma~\ref{lemma_high_prob_cum_reg} then provides a concentration bound for this martingale term. Finally, 
Lemma~\ref{lemma_sum_of_stds_pasztor} combines these results with the expected regret bound from 
Lemma~\ref{lemma_expectation_instant_ub} to obtain a high-probability bound on the cumulative regret.

We start by defining the following quantities.
\begin{align*}
    \Bar{r}_t &= r_t \mathds{1}\{E_1\}\\
    X_t &= \Bar{r}_t - C'L_\mu d_t\sigma_t(x_t,x_t') -\frac{L_\mu+1}{t^2}\\
    Y_t &= \sum_{s=1}^t X_s
\end{align*}

It is then straightforward to show that the process $Y_t$ is a super martingale.
\begin{lemma} \label{lemma_martingale_true}
    The process $(Y_t;t=0, \dots, T)$ is a super martingale process with respect to filtration $\mathcal{F}_{t}$
\end{lemma}
\begin{proof}
    From Lemma \ref{lemma_expectation_instant_ub} we know that 
    \begin{equation*}
        \mathbb{E}\left[\Bar{r}_t | \cF_{t-1}\right]\leq C'L_\mu d_t\mathbb{E}\left[\sigma_t(x_t,x_t')| \cF_{t-1}\right] +\frac{L_\mu+1}{t^2}.
    \end{equation*}
    Hence, we immediately have
    \begin{equation*}
        \mathbb{E}\left[Y_t - Y_{t-1}|\mathcal{F}_{t-1}\right] \leq 0,
    \end{equation*}
    namely the process $(Y_t;t=0, \dots, T)$ is a super martingale.
\end{proof}

\begin{lemma}\label{lemma_high_prob_cum_reg}
     Let $\delta\in (0,1/2)$. With probability greater than $1-2\delta$ the total cumulative regret can be upper bounded by
     \begin{equation*}
         R_T = \sum_{t=1}^t r(t) = \tilde{\mathcal{O}}\left(\beta_t(\delta)\sum_{t=1}^T \sigma_t(x_t,x_t') + \sqrt{\log(1/\delta)T}\right)
     \end{equation*}
\end{lemma}
\begin{proof}
 By definition of $Y_t$ and $X_t$, we have 
    \begin{align*}
        |Y_t - Y_{t-1}|  &= |X_t|\\
        &\leq |\Bar{r}_t| + C'L_\mu d_t\sigma_t(x_t,x_t') +\frac{L_\mu+1}{t^2}\\
        &\leq \frac{1}{2}+ C'L_\mu d_t\sigma_t(x_t,x_t') +\frac{L_\mu+1}{t^2}\\
        &\leq \frac{1}{2}+ C'L_\mu d_t +\frac{L_\mu+1}{t^2}\\
    \end{align*}
    Hence, we can apply the Azuma-Hoeffding inequality \citep[Lemma 11, ][]{chowdhury2017kernelized}     to attain the following with probability greater than $1-\delta$
    \begin{align}
        \sum_{t=1}^T \Bar{r}_T &\leq \sum_{t=1}^T C'L_\mu d_t\sigma_t(x_t,x_t') +\sum_{t=1}^T\frac{L_\mu+1}{t^2} + \sqrt{2\log(1/\delta) \sum_{t=1}^T \left(\frac{1}{2}+ C'L_\mu d_t +\frac{L_\mu+1}{t^2}\right)^2}\nonumber\\
        &\leq C'L_\mu d_T\sum_{t=1}^T \sigma_t(x_t,x_t') +(L_\mu+1)\frac{\pi^2}{6} + \sqrt{2\log(1/\delta)\cdot T\left(\frac{1}{2} + C'L_\mu d_T + L_\mu+1\right)^2}\nonumber\\
        &= C'L_\mu d_T\sum_{t=1}^T \sigma_t(x_t,x_t') +(L_\mu+1)\frac{\pi^2}{6} + \left(\frac{1}{2} + C'L_\mu d_T + L_\mu+1\right)\sqrt{2T\log(1/\delta)}\label{eqn_union_bound_event_endofproof}
    \end{align}
    Where in the second inequality we have used that $d_t \leq d_T$ and calculated the sums. Now, since we have defined $\Bar{r}_t = r_t \mathds{1}\{E_1\}$, we know that under event $E_1$ we have $\Bar{r}_t=r_t$ for all $t \geq 1$. Furthermore, from \eqref{eqn:E1_prob}, we know event $E_1$ holds with probability at least $1-\delta$. Hence by taking a union bound over events $E_1$ and \eqref{eqn_union_bound_event_endofproof}, the following holds with probability greater than $1-2\delta$
    \begin{equation*}
        \sum_{t=1}^T r_t \leq C'L_\mu d_T\sum_{t=1}^T \sigma_t(x_t,x_t') +(L_\mu+1)\frac{\pi^2}{6} + \left(\frac{1}{2} + C'L_\mu d_T + L_\mu+1\right)\sqrt{2T\log(1/\delta)}
    \end{equation*}
    Then plugging in our definition for $d_t$ from \eqref{eqn_d_t_definition} yields the desired result. 
\end{proof}

\begin{lemma}\label{lemma_sum_of_stds_pasztor}
     The sum of posterior standard deviation for the action pairs played from rounds $t=1$ to $T$ can be upper bounded by 
     \begin{equation*}
         \sum_{t=1}^T \sigma_t(x_t,x_t') \leq \sqrt{\frac{8T\,\Gamma(T)}{\log(1+4(\lambda\kappa)^{-1})}}
     \end{equation*}
\end{lemma}

\begin{proof}
    By the Cauchy Schwarz inequality we know that 
    \begin{equation*}
        \sum_{t=1}^T \sigma_t(x_t,x_t') \leq \sqrt{T\sum_{t=1}^T \sigma_t^2(x_t,x_t')}.
    \end{equation*}
    Note that from \citep[Lemma 14][]{pasztor2024bandits}, we have the following inequality for the sum of variances. 
    \begin{equation*}
        \sum_{t=1}^T \sigma_t^2(x_t,x_t') \leq \frac{8\,\Gamma(T)}{\log(1+4(\lambda\kappa)^{-1})}
    \end{equation*}
    Plugging this in to our first equation gives us the desired result.
\end{proof}

Finally, combining the bound on the sum of posterior standard deviations from 
Lemma~\ref{lemma_sum_of_stds_pasztor} with the high-probability bound on the cumulative 
regret established in Lemma~\ref{lemma_high_prob_cum_reg} yields the desired regret 
guarantee for Theorem~\ref{theorem_regret}. In particular, with probability at least 
$1-2\delta$, the cumulative regret satisfies
\begin{align*}
    R_T = \tilde{\mathcal{O}}\!\left(\beta_T(\delta)\sqrt{T\Gamma(T)} + \sqrt{T\log(1/\delta)}\right).
\end{align*}
This completes the proof of Theorem~\ref{theorem_regret}.

\section{Additional Experiments}\label{appendix:additional_experiments}

In this section, we present additional experiments that further support the findings of the main paper. In addition to the common Bayesian optimization benchmark and real-world catalyst design task considered in Section~\ref{sec:experiments}, we evaluate the proposed method on two new objectives that capture complementary aspects of preference-based optimization.

The first objective is a synthetic function drawn from the reproducing kernel Hilbert space (RKHS) of a Matérn-$5/2$ kernel. The function is generated following the procedure of \citet{kayal2025bayesian} and provides a controlled setting that
aligns closely with the kernelized assumptions underlying BOHF methods. The second objective is a hyperparameter tuning task involving the optimization of five hyperparameters (batch size, number of units per layer, learning rate, momentum, and weight decay) for a four-layer neural network trained on the MNIST classification task, using data from the LCBench benchmark \citep{zimmer2021autopytorch}. This setting models noisy preference feedback obtained from partial training runs and reflects a realistic scenario for scalable hyperparameter optimization.

Across both objectives, we observe trends consistent with the results reported for the Catalyst and Ackley tasks. In particular, as shown in Figure~\ref{fig:rebuttal}, PF-TS exhibits a faster reduction in instantaneous regret during the early stages of optimization compared to MR-LPF, which leads to a substantial improvement in cumulative regret after 200 iterations. Final average cumulative regret values, reported in the rebuttal, further confirm that PF-TS achieves competitive or improved performance relative to existing methods on both synthetic and practical preference-based optimization problems.



\begin{figure*}[ht]
\centering
\begin{subfigure}{0.49\linewidth}
\centering
    \includegraphics[width=0.7\linewidth]{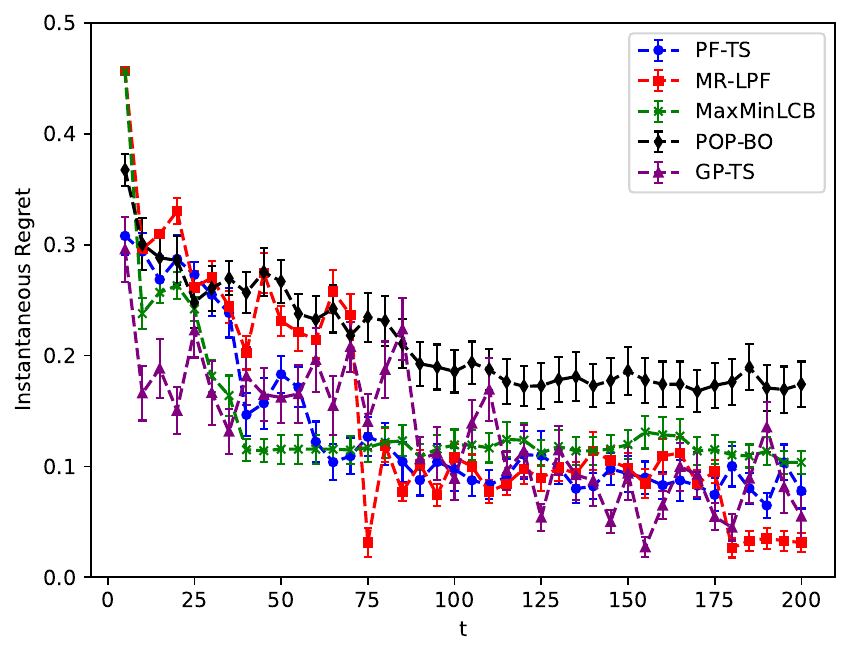}
        \caption{\textbf{(HPO)} Instantaneous Regret}
    \label{fig:hpo_1}
\end{subfigure}
\begin{subfigure}{0.49\linewidth}
\centering
    \includegraphics[width=0.7\linewidth]{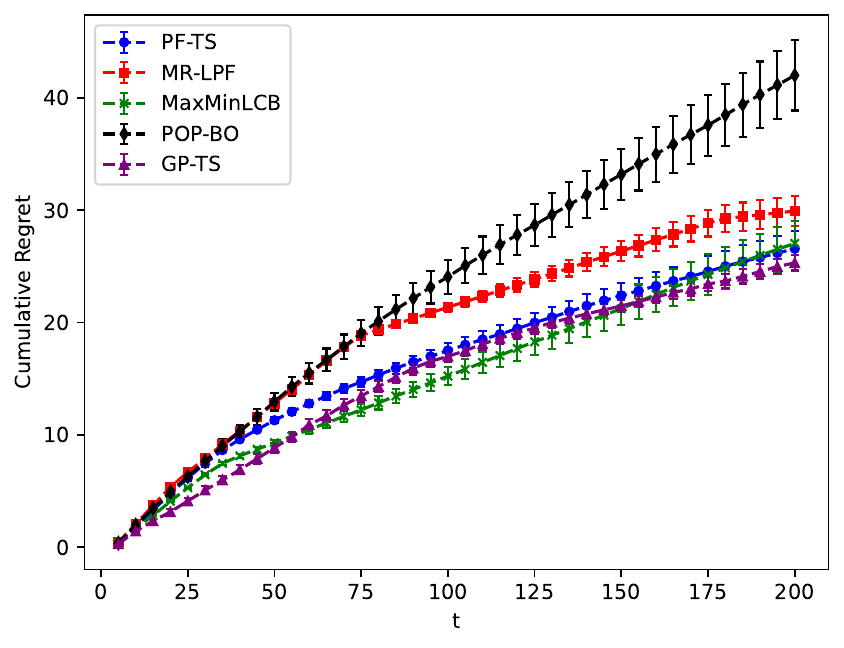}
        \caption{\textbf{(HPO)} Cumulative Regret}
    \label{fig:hpo_2}
\end{subfigure}
\begin{subfigure}{0.49\linewidth}
\centering
    \includegraphics[width=0.7\linewidth]{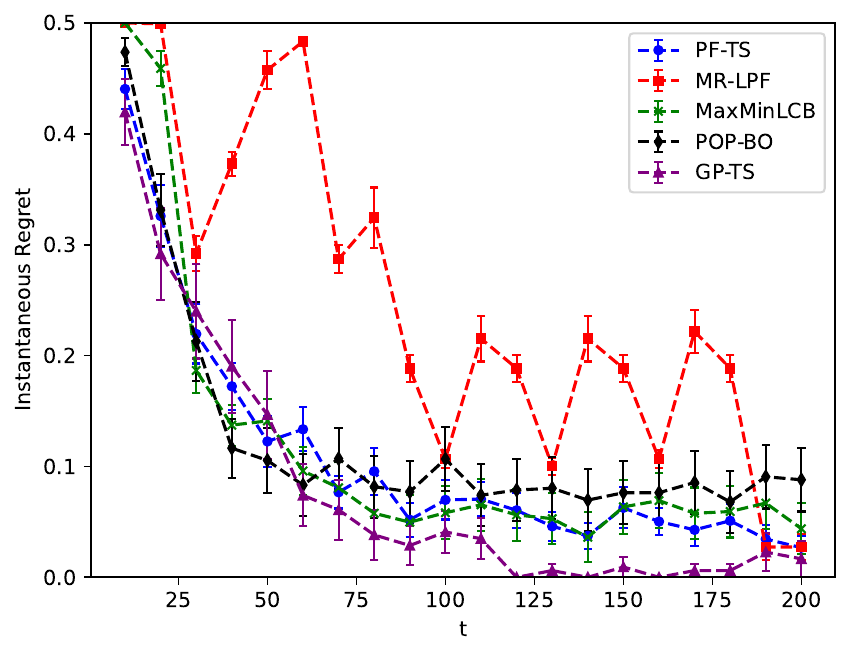}
        \caption{\textbf{(RKHS)} Instantaneous Regret}
    \label{fig:rkhs_1}
\end{subfigure}
\begin{subfigure}{0.49\linewidth}
\centering
    \includegraphics[width=0.7\linewidth]{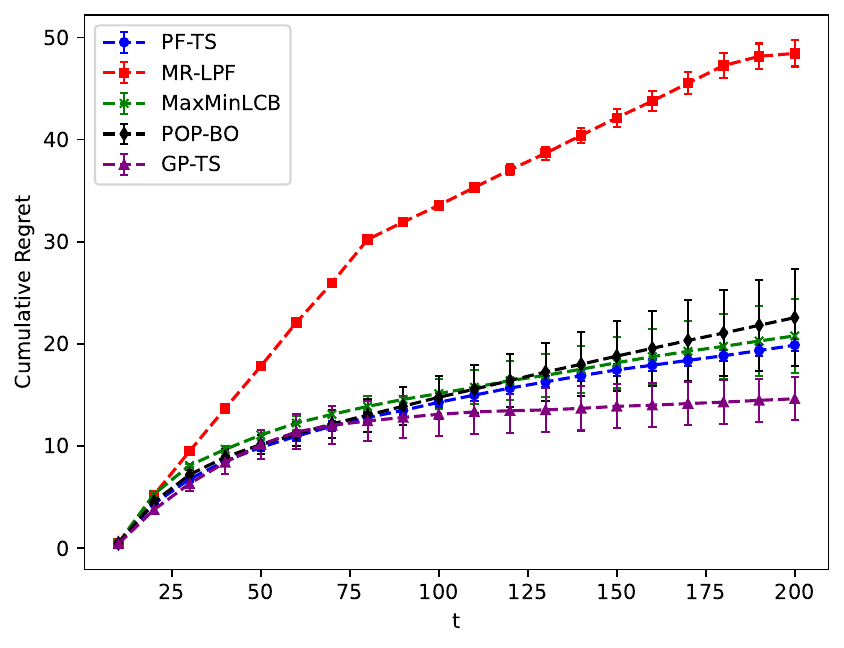}
        \caption{\textbf{(RKHS)} Cumulative Regret}
    \label{fig:rkhs_2}
\end{subfigure}
\caption{Regret of PF-TS, MR-LPF, MaxMinLCB, POP-BO, and GP-TS. Mean over 30 runs; shaded bands show $\pm$1 standard error. (a,b) Hyperparameter tuning objective ($T=200$). (c,d) RKHS optimisation objective ($T=200$). Left column: instantaneous (simple) regret. Right column: cumulative regret~\eqref{regret_definition_duel}.}

\label{fig:rebuttal}
\vskip -0.2in
\end{figure*}

\section{Experimental Details }\label{appendix:experiment}

In this section we provide more details of the experiments run and displayed in the main paper. We also provide any additional, supporting results.

\subsection{Ackley Function Experiments} \label{app_sub_details_ackley}

The Ackley function considered in Section \ref{sec:experiments} is defined as 

\begin{equation} \label{eq:ackley_definition}
    \text{Ackley}(x) = -a \exp\left(-b \sqrt{\frac{1}{d}\sum_{i=1}^d x_i^2}\right) - \exp\left(\frac{1}{d}\sum_{i=1}^d \cos (cx_i)\right) + a + \exp(1)
\end{equation}
with constants $a=20$, $b=0.2$, $c=2\pi$ and $d$ is the number of dimensions in the action space.

In Bayesian Optimization we oftentimes try to minimize the function, however in this problem formulation we try to maximize the utility. Hence, in our implementation, we flip the Ackley function and try to find the maxima. This is illustrated in Figure \ref{eq:ackley_definition} below and was what we used in the experiments. 

\begin{figure}[H]\label{fig:illustration_ackley}
    \centering
    \includegraphics[width=0.5\linewidth]{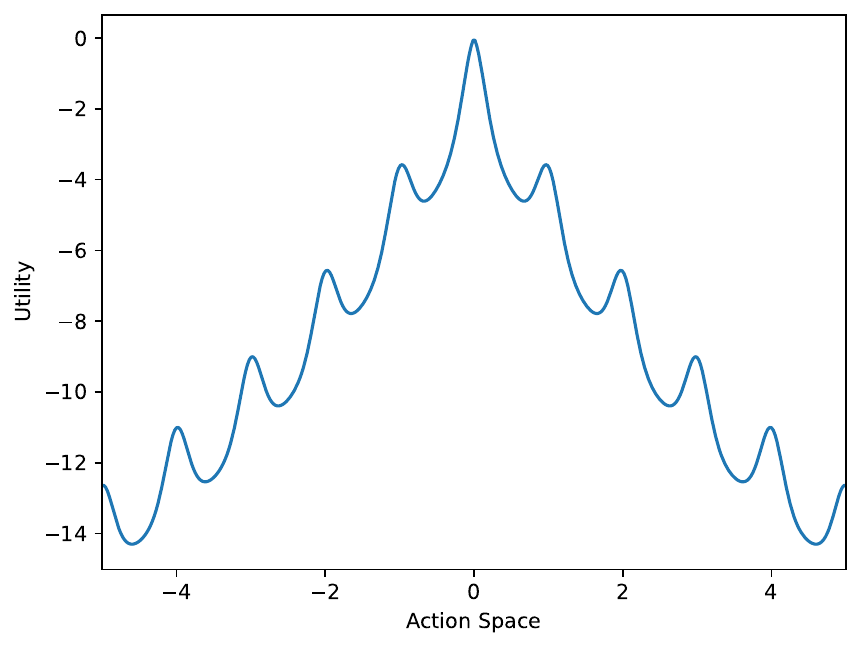}
    \caption{Visualization of Ackley utility function used in experiments.}
    \label{fig:placeholder}
\end{figure}

For our simulations, we consider the one-dimensional ackley function and discretize the action space $\mathcal{X}=[-5,5]$ into 40 evenly spaced points from which the algorithms can pick any two actions $x_t,x_t'$ to play in each iteration. To produce the preference feedback from this utility function, we simply draw a Bernoulli random variable as 

$$y_t = \mathds{1}\{x_t\succ x_t'\} \sim \text{Bernoulli}(\sigma(f(x_t)-f(x'_t)).$$

In Figures \ref{fig:ackley_1} and \ref{fig:ackley_2} we plot the regret from using algorithms PF-TS (Ours), MR-LPF \citep{kayal2025bayesian}, MaxMinLCB \citep{pasztor2024bandits} and POP-BO \citep{xu2024principled}. We use the same inference described in Section \ref{prediction_uncertainty} with regularization term $\lambda=0.05$ and we have used a M\'atern 2.5 kernel for each algorithm with length scale $\ell=0.1$. In Figure \ref{fig:ackley_1} we plot the mean instantaneous regret and in Figure \ref{fig:ackley_2} we plot the mean cumulative regret (as well as the standard errors on each) at each iteration from 30 runs of each algorithm up to horizon $T=300$. On the same figure we have also plotted GP-TS \citep{chowdhury2017kernelized} with access to direct scalar feedback of the same utility function $f$, where we use standard Gaussian noise for the observations. Namely, we allow GP-TS in each round to select one action from the action space $x_t \in \mathcal{X}$. From this action they then observe direct scalar feedback of
$$y_t = f(x_t) + \epsilon_t,$$
where $\epsilon_t \sim N(0,1)$. Then we record the regret in each round $t$ to be the same as the preference feedback in equation \eqref{regret_definition_duel} with $x'_t=x_t$. Again, this is equivalent to using the standard regret definition of $f(x^*) - f(x_t)$ up to the same constant. By using \eqref{regret_definition_duel} we simply put the regret for all algorithms on the same scale. For confidence interval-based algorithms MR-LPF, MaxMinLCB, POP-BO, and GP-UCB we used the confidence interval width coefficient equal to 1 \citep[as in ][]{kayal2025bayesian}. For GP-TS and TS we used the exploration bonus $v_t^2 = \sqrt{t+1+\log(2/0.05)}$. This performs well in practice and is a minor modification to our bonus in \eqref{eq:beta_def}, where we have loosened the information gain to $t$ and then taken the square-root.

\subsection{Catalyst Composition Experiments}

The data used for the experiments described in Section \ref{subsection:catalyst_experiments} can be found in the {Open Catalyst Experiments 2024 (OCx24)} dataset released by \emph{FAIR Chemistry} (Meta AI) \citep{abed2024open}. The full dataset can be accessed publically on github  (\url{https://github.com/facebookresearch/fairchem/tree/main/src/fairchem/applications/ocx/data/experimental_data}) which reports hydrogen yields for various different catalyst compositions, for two different experiments at different currents. One of the most popular compositions studied in this work are catalysts composed of silver (Ag), gold (Au), and zinc (Zn). Hence, to define our utility function, we fix a particular experimental set up and only vary these compositions. In particular, we consider the CO2R reaction at a 300-amp current. We then record the hydrogen yield ('fe-h2') for each of the 63 such compositions as the utility for each Ag/Au/Zn composition. This extracted data be viewed in the supplementary material. The hydrogen yields, originally measured quantitatively in the range [0,100], were rescaled to [0,10] for the purposes of our experiments.

Figures~\ref{fig:cat_1} and~\ref{fig:cat_2} similarly plot regret for PF-TS (ours), MR-LPF \citep{kayal2025bayesian}, MaxMinLCB \citep{pasztor2024bandits}, and POP-BO \citep{xu2024principled}, using the same inference procedure as Section~\ref{prediction_uncertainty} and a Mat\'ern-$5/2$ kernel with hyperparameters same as the previous experiments. For reference, we again also include cumulative regret for vanilla GP-TS \citep{chowdhury2017kernelized} with direct scalar feedback on $f$. Results are averaged over 30 runs up to a horizon of $T=800$ iterations.

\subsection{RKHS Function Experiments}
For experimental results displayed in Figures \ref{fig:rkhs_1} and \ref{fig:rkhs_2} we use a utility function sampled from an RKHS with Matérn-$5/2$ kernel (similar to the process used by \citet{kayal2025bayesian}). The sampled utility function is displayed in the following figure. When running the experiments using this utility function, we discretize the $[0,1]$ action space into 30 evenly spaced points. We then construct any preference feedback from action pairs with the same procedure discussed in Appendix \ref{app_sub_details_ackley}.

\begin{figure}[H]\label{fig:illustration_ackley}
    \centering
    \includegraphics[width=0.5\linewidth]{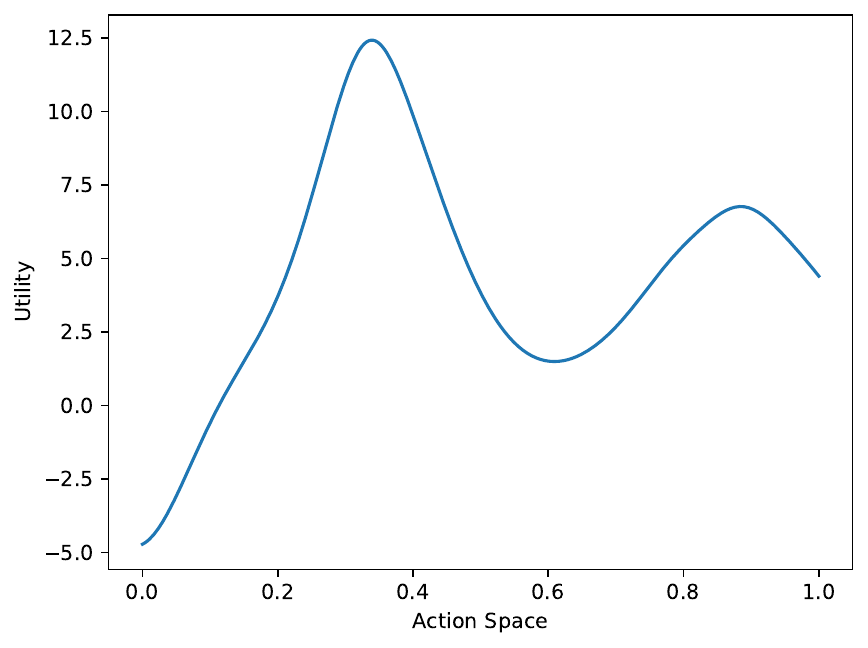}
    \caption{Visualization of utility function sampled from RKHS used in experiments.}
    \label{fig:rebuttal_RKHS_experiments_utility_function}
\end{figure}

\subsection{Hyperparameter Tuning Experiments}

For the results displayed in Figures \ref{fig:hpo_1} and \ref{fig:hpo_2}, we use data from the hyperparamter tuning benchmark LCBench \citep{zimmer2021autopytorch}. In particular, we look at the recorded accuracy data for different neural network architectures with the objective of image classification on the MNIST dataset. For these experiments we consider a set of 40 hyperparameter configurations for the neural networks used which have 4 layers. Hence, in this case we define a utility function $f$ to be the accuracy achieved from each configuration and the action space $\mathcal{X}$ is 5-dimensional with entries representing batch size, number of units per layer, learning rate, momentum, and weight decay. With this utility function and action space defined, we construct the preference feedback in the same way as usual (see Appendix \ref{app_sub_details_ackley}).

\subsection{Cost of Scalar vs Preference Feedback}

For the results displayed in Figure~\ref{fig:cat_cost}, we use the results from the experiments from the previous Section \ref{subsection:catalyst_experiments}. In particular, we take the instantaneous regret data for PF-TS (ours) using preference feedback and vanilla GP-TS which has access to direct scalar feedback from the catalyst composition utility function over a horizon of $T=800$ iterations. We then plot a total cost budget $c$ versus achieved instantaneous regret for our algorithm PF-TS, where one unit of cost is the cost of one preference feedback iteration (i.e. cost budget $c$ buys $c$ preference feedback iterations). We then consider various cost ratios $\xi\in\{1,3,5,7\}$ between scalar and preference feedback (i.e. cost budget $c$ buys $c/\xi$ scalar feedback iterations). Under these different ratios, we plot the total cost budget $c$ versus achieved instantaneous regret for the vanilla GP-TS algorithm using scalar feedback. Namely, at each cost budget $c$, PF-TS runs for $c$ preference iterations, while GP-TS is allotted only $c/\xi$ scalar iterations. In Figure \ref{fig:cat_cost}, we plot the regret and $\pm$1 standard error after each additional 25 iterations can be bought by each algorithm at each ratio. Note that we see that with a cost budget of $c=800$ we can afford many more preference feedback iterations than scalar feedback iterations when the cost ratio, $\xi$, is high for the scalar feedback. Hence, we can achieve lower instantaneous regret in these settings.

\subsection{Anchor Independence}

The plot in Figure \ref{fig:anchor} was made by considering an action space of size $|\mathcal{X}|=50$ and an Ackley utility function. We uniformly select 200 pairs of actions and then estimate $h_t$, $\sigma_t$ using the inference presented in Section \ref{sec:prelim}. We then draw a posterior sample of $\Tilde{h}_t$ as

$$\Tilde{h}_t(\cdot,\cdot) \sim GP(h_t, v_t^2 k_t^\Delta).$$ 

Then for each of 25 evenly spaced choices for $x_0 \in [0,5]$, we plot the maximizer of the single-argument search for TS as 
$$x_t = \argmax_{x \in \mathcal{X}} \Tilde{h}_t(x,x_0)$$
and for UCB
$$x_t = \argmax_{x \in \mathcal{X}} h_t(x,x_0) + \sigma_t^\Delta(x,x_0).$$
Here we have arbitrarily selected the scaling parameter $\beta=1$ for UCB in order to illustrate it's anchor dependence. To reiterate: for a Thompson selected action, we see from Figure \ref{fig:anchor} that any choice of anchor point leads to the same choice of action. On the other hand, the action selected by UCB can vary dependent on the choice of anchor point.

\subsection{Compute Information}

For the implementation of the inference methods described in Section~\ref{sec:prelim}, we extended the recent codebase of \citet{kayal2025bayesian}. The full implementation of our method is provided in the supplementary material. All simulations were conducted on a workstation equipped with an NVIDIA GeForce RTX~3090 GPU (25\,GB RAM) and an AMD EPYC~7542 32-core CPU. 

For each preference-feedback algorithm applied to the Ackley utility function (Section~\ref{subsection:ackley_experiments}), the runtime per experiment was approximately 5--10 hours. For the catalyst dataset experiments (Section~\ref{subsection:catalyst_experiments}), each preference-feedback algorithm required approximately 10--20 hours to complete.

%% file: CAMERA_READY/AISTATS_camera_ready/bibliography.bib
@inproceedings{christiano2017deep,
  title     = {Deep Reinforcement Learning from Human Preferences},
  author    = {Christiano, Paul F. and Leike, Jan and Brown, Tom B. and Martic, Miljan and Legg, Shane and Amodei, Dario},
  booktitle = {Advances in Neural Information Processing Systems (NeurIPS)},
  year      = {2017}
}

@inproceedings{ouyang2022training,
  title     = {Training Language Models to Follow Instructions with Human Feedback},
  author    = {Ouyang, Long and Wu, Jeff and Jiang, Xu and Almeida, Diogo and Wainwright, Carroll L. and Mishkin, Pamela and Zhang, Chong and Agarwal, Sandhini and Slama, Katarina and Ray, Alex and Schulman, John and Hilton, Jacob and Kelton, Fraser and Miller, Luke and Simens, Maddie and Askell, Amanda and Welinder, Peter and Christiano, Paul and Leike, Jan and Lowe, Ryan},
  booktitle = {Advances in Neural Information Processing Systems (NeurIPS)},
  year      = {2022}
}

@article{bai2022constitutional,
  title   = {Constitutional {AI}: Harmlessness from {AI} Feedback},
  author  = {Bai, Yuntao and Kadavath, Saurav and Kundu, Sandipan and Askell, Amanda and Kernion, Jackson and Jones, Andy and Chen, Anna and Goldie, Anna and Mirhoseini, Azalia and McKinnon, Cameron and Chen, Carol and Olsson, Catherine and Olah, Christopher and Hernandez, Danny and Drain, Dawn and Ganguli, Deep and Li, Dustin and Tran-Johnson, Eli and Perez, Ethan and Kerr, Jamie and Mueller, Jared and Ladish, Jeffrey and Landau, Joshua and Ndousse, Kamal and Lukosuite, Kamile and Lovitt, Liane and Sellitto, Michael and Elhage, Nelson and Schiefer, Nicholas and Mercado, Noemi and DasSarma, Nova and Lasenby, Robert and Larson, Robin and Ringer, Sam and Johnston, Scott and Kravec, Shauna and El Showk, Sheer and Fort, Stanislav and Lanham, Tamera and Telleen-Lawton, Timothy and Conerly, Tom and Henighan, Tom and Hume, Tristan and Bowman, Samuel R. and Hatfield-Dodds, Zac and Mann, Ben and Amodei, Dario and Joseph, Nicholas and McCandlish, Sam and Brown, Tom and Kaplan, Jared},
  journal = {arXiv preprint arXiv:2212.08073},
  year    = {2022}
}

@inproceedings{liu2023geval,
  title     = {G-Eval: {NLG} Evaluation using {GPT-4} with Better Human Alignment},
  author    = {Liu, Yang and Iter, Dan and Xu, Yichong and Wang, Shuohang and Xu, Ruochen and Zhu, Chenguang},
  booktitle = {Proceedings of the 2023 Conference on Empirical Methods in Natural Language Processing (EMNLP)},
  year      = {2023},
  publisher = {Association for Computational Linguistics}
}

@inproceedings{yue2009interactively,
  title={Interactively optimizing information retrieval systems as a dueling bandits problem},
  author={Yue, Yisong and Joachims, Thorsten},
  booktitle={Proceedings of the 26th Annual International Conference on Machine Learning},
  pages={1201--1208},
  year={2009}
}

@inproceedings{srinivas2009gaussian,
  title={Gaussian process optimization in the bandit setting: No regret and experimental design},
  author={Srinivas, Niranjan and Krause, Andreas and Kakade, Sham M and Seeger, Matthias},
  booktitle={International Conference on Machine Learning},
  year={2010}
}

@article{bradley1952rank,
  title={Rank analysis of incomplete block designs: I. The method of paired comparisons},
  author={Bradley, Ralph Allan and Terry, Milton E},
  journal={Biometrika},
  year={1952}
}

@article{bengs2021preference,
  title={Preference-based online learning with dueling bandits: A survey},
  author={Bengs, Viktor and Busa-Fekete, R{\'o}bert and El Mesaoudi-Paul, Adil and H{\"u}llermeier, Eyke},
  journal={Journal of Machine Learning Research},
  year={2021}
}

@inproceedings{gonzalez2017preferential,
  title={Preferential bayesian optimization},
  author={Gonz{\'a}lez, Javier and Dai, Zhenwen and Damianou, Andreas and Lawrence, Neil D},
  booktitle={International Conference on Machine Learning},
  year={2017}
}

@inproceedings{mikkola2020projective,
  title={Projective preferential bayesian optimization},
  author={Mikkola, Petrus and Todorovi{\'c}, Milica and J{\"a}rvi, Jari and Rinke, Patrick and Kaski, Samuel},
  booktitle={International Conference on Machine Learning},
  year={2020}
}

@inproceedings{takeno2023towards,
  title={Towards practical preferential Bayesian optimization with skew Gaussian processes},
  author={Takeno, Shion and Nomura, Masahiro and Karasuyama, Masayuki},
  booktitle={International Conference on Machine Learning},
  year={2023}
}

@inproceedings{xu2024principled,
  title={Principled Preferential Bayesian Optimization},
  author={Xu, Wenjie and Wang, Wenbin and Jiang, Yuning and Svetozarevic, Bratislav and Jones, Colin},
  booktitle={International Conference on Machine Learning},
  year={2024}
}

@article{abbasi2013online,
author = {Abbasi-Yadkori, Yasin},
title = {Online learning for linearly parametrized control problems},
year = {2013},
journal = {University of Alberta}
}

@inproceedings{chowdhury2017kernelized,
  title={On kernelized multi-armed bandits},
  author={Chowdhury, Sayak Ray and Gopalan, Aditya},
  booktitle={International Conference on Machine Learning},
  year={2017}
}

@article{vakili2021optimal,
  title={Optimal order simple regret for Gaussian process bandits},
  author={Vakili, Sattar and Bouziani, Nacime and Jalali, Sepehr and Bernacchia, Alberto and Shiu, Da-shan},
  journal={Advances in Neural Information Processing Systems},
  year={2021}
}

@inproceedings{scarlett2017lower,
  title={Lower bounds on regret for noisy gaussian process bandit optimization},
  author={Scarlett, Jonathan and Bogunovic, Ilija and Cevher, Volkan},
  booktitle={Conference on Learning Theory},
  year={2017}
}

@article{whitehouse2024sublinear,
  title={On the sublinear regret of GP-UCB},
  author={Whitehouse, Justin and Ramdas, Aaditya and Wu, Steven Z},
  journal={Advances in Neural Information Processing Systems},
  volume={36},
  year={2024}
}

@article{yue2012k,
  title={The k-armed dueling bandits problem},
  author={Yue, Yisong and Broder, Josef and Kleinberg, Robert and Joachims, Thorsten},
  journal={Journal of Computer and System Sciences},
  year={2012},
  publisher={Elsevier}
}

@article{saha2021optimal,
  title={Optimal algorithms for stochastic contextual preference bandits},
  author={Saha, Aadirupa},
  journal={Advances in Neural Information Processing Systems},
  year={2021}
}

@inproceedings{saha2022efficient,
  title={Efficient and optimal algorithms for contextual dueling bandits under realizability},
  author={Saha, Aadirupa and Krishnamurthy, Akshay},
  booktitle={International Conference on Algorithmic Learning Theory},
  year={2022}
}

@inproceedings{dudik2015contextual,
  title={Contextual dueling bandits},
  author={Dud{\'\i}k, Miroslav and Hofmann, Katja and Schapire, Robert E and Slivkins, Aleksandrs and Zoghi, Masrour},
  booktitle={Conference on Learning Theory},
  year={2015}
}

@inproceedings{
wu2023making,
title={Making {RL} with Preference-based Feedback Efficient via Randomization},
author={Runzhe Wu and Wen Sun},
booktitle={The Twelfth International Conference on Learning Representations},
year={2024}
}

@inproceedings{xu2020zeroth,
  title={Zeroth order non-convex optimization with dueling-choice bandits},
  author={Xu, Yichong and Joshi, Aparna and Singh, Aarti and Dubrawski, Artur},
  booktitle={Conference on Uncertainty in Artificial Intelligence},
  year={2020}
}

@article{jamil2013literature,
  title={A literature survey of benchmark functions for global optimisation problems},
  author={Jamil, Momin and Yang, Xin-She},
  journal={International Journal of Mathematical Modelling and Numerical Optimisation},
  year={2013}
}

@article{thompson1933likelihood,
  title={On the likelihood that one unknown probability exceeds another in view of the evidence of two samples},
  author={Thompson, William R},
  journal={Biometrika},
  year={1933}
}

@inproceedings{valko2013finite,
  title={Finite-Time Analysis of Kernelised Contextual Bandits},
  author={Valko, Michal and Korda, Nathan and Munos, R{\'e}mi and Flaounas, Ilias and Cristianini, Nello},
  booktitle={Uncertainty in Artificial Intelligence},
  year={2013}
}

@article{salgia2021domain,
  title={A domain-shrinking based bayesian optimization algorithm with order-optimal regret performance},
  author={Salgia, Sudeep and Vakili, Sattar and Zhao, Qing},
  journal={Advances in Neural Information Processing Systems},
  year={2021}
}

@inproceedings{
zhan2023provable,
title={Provable Offline Preference-Based Reinforcement Learning},
author={Wenhao Zhan and Masatoshi Uehara and Nathan Kallus and Jason D. Lee and Wen Sun},
booktitle={The Twelfth International Conference on Learning Representations},
year={2024}
}

@inproceedings{pasztor2024bandits,
title={Bandits with Preference Feedback: A Stackelberg Game Perspective},
author={Barna P{\'a}sztor and Parnian Kassraie and Andreas Krause},
booktitle={Advances in Neural Information Processing Systems 38},
year={2024}
}

@inproceedings{lifeel,
title={Feel-Good Thompson Sampling for Contextual Dueling Bandits},
author={Li, Xuheng and Zhao, Heyang and Gu, Quanquan},
booktitle={Forty-first International Conference on Machine Learning},
year={2024}
}

@inproceedings{bengs2022stochastic,
  title={Stochastic contextual dueling bandits under linear stochastic transitivity models},
  author={Bengs, Viktor and Saha, Aadirupa and H{\"u}llermeier, Eyke},
  booktitle={International Conference on Machine Learning},
  year={2022}
}

@InProceedings{pmlr-v151-li22a,
  title = 	 { Gaussian Process Bandit Optimization with Few Batches },
  author =       {Li, Zihan and Scarlett, Jonathan},
  booktitle = 	 {Proceedings of The 25th International Conference on Artificial Intelligence and Statistics},
  year = 	 {2022}
}

@inproceedings{kayal2025bayesian,
  title        = {Bayesian Optimization from Human Feedback: Near-Optimal Regret Bounds},
  author       = {Kayal, Aya and Vakili, Sattar and Toni, Laura and Shiu, Da-shan and Bernacchia, Alberto},
  booktitle    = {Proceedings of the 42nd International Conference on Machine Learning},
  year         = {2025}
}

@inproceedings{agrawal2012analysis,
  title     = {Analysis of Thompson Sampling for the Multi-armed Bandit Problem},
  author    = {Agrawal, Shipra and Goyal, Navin},
  booktitle = {Proceedings of the 25th Annual Conference on Learning Theory},
  year      = {2012}
}

@inproceedings{kaufmann2012thompson,
  title     = {Thompson Sampling: An Asymptotically Optimal Finite-Time Analysis},
  author    = {Kaufmann, {\'E}milie and Korda, Nathaniel and Munos, R{\'e}mi},
  booktitle = {Proceedings of the 23rd International Conference on Algorithmic Learning Theory},
  year      = {2012}
}

@article{russo2014learning,
  title   = {Learning to Optimize via Posterior Sampling},
  author  = {Russo, Daniel and Van Roy, Benjamin},
  journal = {Mathematics of Operations Research},
  year    = {2014}
}

@article{abed2024open,
  title={Open catalyst experiments 2024 (OCx24): Bridging experiments and computational models},
  author={Abed, Jehad and Kim, Jiheon and Shuaibi, Muhammed and Wander, Brook and Duijf, Boris and Mahesh, Suhas and Lee, Hyeonseok and Gharakhanyan, Vahe and Hoogland, Sjoerd and Irtem, Erdem and others},
  journal={arXiv preprint arXiv:2411.11783},
  year={2024}
}

@inproceedings{
mehta2025sample,
title={Sample Efficient Preference Alignment in {LLM}s via Active Exploration},
author={Viraj Mehta and Syrine Belakaria and Vikramjeet Das and Ojash Neopane and Yijia Dai and Ilija Bogunovic and Barbara E Engelhardt and Stefano Ermon and Jeff Schneider and Willie Neiswanger},
booktitle={Second Conference on Language Modeling},
year={2025}
}

@article{wang2023recent,
  title={Recent advances in Bayesian optimization},
  author={Wang, Xilu and Jin, Yaochu and Schmitt, Sebastian and Olhofer, Markus},
  journal={ACM Computing Surveys},
  year={2023}
}

@article{EI_1998,
author = {Jones, Donald and Schonlau, Matthias and Welch, William},
year = {1998},
title = {Efficient Global Optimization of Expensive Black-Box Functions},
journal = {Journal of Global Optimization}
}

@InProceedings{pmlr-v151-tran-the22a,
  title = 	 { Regret Bounds for Expected Improvement Algorithms in Gaussian Process Bandit Optimization },
  author =       {Tran-The, Hung and Gupta, Sunil and Rana, Santu and Venkatesh, Svetha},
  booktitle = 	 {Proceedings of The 25th International Conference on Artificial Intelligence and Statistics},
  year = 	 {2022},
}

@article{wu2016double,
  title={Double thompson sampling for dueling bandits},
  author={Wu, Huasen and Liu, Xin},
  journal={Advances in neural information processing systems},
  year={2016}
}

@inproceedings{siivola2021preferential,
  title={Preferential batch Bayesian optimization},
  author={Siivola, Eero and Dhaka, Akash Kumar and Andersen, Michael Riis and Gonz{\'a}lez, Javier and Moreno, Pablo Garc{\'\i}a and Vehtari, Aki},
  booktitle={2021 IEEE 31st International Workshop on Machine Learning for Signal Processing (MLSP)},
  year={2021},
  organization={IEEE}
}

@InProceedings{pmlr-v32-zoghi14,
  title = 	 {Relative Upper Confidence Bound for the K-Armed Dueling Bandit Problem},
  author = 	 {Zoghi, Masrour and Whiteson, Shimon and Munos, Remi and Rijke, Maarten},
  booktitle = 	 {Proceedings of the 31st International Conference on Machine Learning},
  year = 	 {2014}
}

@InProceedings{pmlr-v32-ailon14,
  title = 	 {Reducing Dueling Bandits to Cardinal Bandits},
  author = 	 {Ailon, Nir and Karnin, Zohar and Joachims, Thorsten},
  booktitle = 	 {Proceedings of the 31st International Conference on Machine Learning},
  year = 	 {2014},
  series = 	 {Proceedings of Machine Learning Research}
}

@article{zoghi2015copeland,
  title={Copeland dueling bandits},
  author={Zoghi, Masrour and Karnin, Zohar S and Whiteson, Shimon and De Rijke, Maarten},
  journal={Advances in neural information processing systems},
  year={2015}
}

@inproceedings{jamieson2015sparse,
  title={Sparse dueling bandits},
  author={Jamieson, Kevin and Katariya, Sumeet and Deshpande, Atul and Nowak, Robert},
  booktitle={Artificial Intelligence and Statistics},
  year={2015}
}

@article{ramamohan2016dueling,
  title={Dueling bandits: Beyond condorcet winners to general tournament solutions},
  author={Ramamohan, Siddartha Y and Rajkumar, Arun and Agarwal, Shivani},
  journal={Advances in Neural Information Processing Systems},
  volume={29},
  year={2016}
}

@InProceedings{pmlr-v139-saha21b,
  title = 	 {Dueling Convex Optimization},
  author =       {Saha, Aadirupa and Koren, Tomer and Mansour, Yishay},
  booktitle = 	 {Proceedings of the 38th International Conference on Machine Learning},
  year = 	 {2021}
}

@inproceedings{dang2025preferential,
  title={Preferential Multi-Objective Bayesian Optimization for Drug Discovery},
  author={Dang, Tai and Pham, Long-Hung and Truong, Sang T and Glenn, Ari and Nguyen, Wendy and Pham, Edward A and Glenn, Jeffrey S and Koyejo, Sanmi and Luong, Thang},
  booktitle={ICLR 2025 Workshop on Human-AI Coevolution},
year={2025}
}

@article{vakili2021scalable,
  title={Scalable Thompson sampling using sparse Gaussian process models},
  author={Vakili, Sattar and Moss, Henry and Artemev, Artem and Dutordoir, Vincent and Picheny, Victor},
  journal={Advances in neural information processing systems},
  year={2021}
}

@article{bull2011convergence,
  title={Convergence rates of efficient global optimization algorithms},
  author={Bull, Adam D},
  journal={Journal of Machine Learning Research},
  year={2011}
}

@inproceedings{wang2014theoretical,
  title={Theoretical analysis of Bayesian optimisation with unknown Gaussian process hyper-parameters},
  author={Wang, Zi and de Freitas, Nando},
  booktitle={Advances in Neural Information Processing Systems},
  year={2014}
}

@inproceedings{tucker2020preference,
  title={Preference-based learning for exoskeleton gait optimization},
  author={Tucker, Maegan and Novoseller, Ellen and Kann, Claudia and Sui, Yanan and Yue, Yisong and Burdick, Joel W and Ames, Aaron D},
  booktitle={2020 IEEE international conference on robotics and automation (ICRA)},
  year={2020},
  organization={IEEE}
}

@article{kushner1964new,
  title   = {A New Method of Locating the Maximum Point of an Arbitrary Multipeak Curve in the Presence of Noise},
  author  = {Kushner, Harold J.},
  journal = {Journal of Basic Engineering},
  year    = {1964}
}

@inproceedings{zheng2023judging,
  title     = {Judging {LLM}-as-a-Judge with {MT}-Bench and Chatbot Arena},
  author    = {Zheng, Lianmin and Chiang, Weize and Sheng, Ying and Li, Zhuohan and Du, Nan and Zhang, Tianjun and Ma, Xuezhe and Abbeel, Pieter and Stoica, Ion and Zhuang, Yong and Li, Eric P. and Lin, Dahua},
  booktitle = {Advances in Neural Information Processing Systems},
  year      = {2023},
  note      = {arXiv:2306.05685}
}

@inproceedings{lee2024rlaif,
  title     = {{RLAIF} vs. {RLHF}: Scaling Reinforcement Learning from Human Feedback with {AI} Feedback},
  author    = {Lee, Harrison and Phatale, Samrat and Mansoor, Hassan and Mesnard, Thomas and Ferret, Johan and Lu, Kellie Ren and Bishop, Colton and Hall, Ethan and Carbune, Victor and Rastogi, Abhinav and Prakash, Sushant},
  booktitle = {Proceedings of the 41st International Conference on Machine Learning},
  year      = {2024}
}

@inproceedings{liu2024pairs,
  title     = {Aligning with Human Judgement: The Role of Pairwise Preference in Large Language Model Evaluators},
  author    = {Liu, Yinhong and Zhou, Han and Guo, Zhijiang and Shareghi, Ehsan and Vuli{\'c}, Ivan and Korhonen, Anna and Collier, Nigel},
  booktitle = {Conference on Language Modeling (COLM)},
  year      = {2024}
}

@article{yang2019machine,
  title={Machine-learning-guided directed evolution for protein engineering},
  author={Yang, Kevin K and Wu, Zachary and Arnold, Frances H},
  journal={Nature methods},
  year={2019}
}

@article{hoffman2013exploiting,
  title={Exploiting correlation and budget constraints in bayesian multi-armed bandit optimization},
  author={Hoffman, Matthew W and Shahriari, Bobak and de Freitas, Nando},
  journal={arXiv preprint arXiv:1303.6746},
  year={2013}
}

@ARTICLE{zimmer2021autopytorch,
  author={Zimmer, Lucas and Lindauer, Marius and Hutter, Frank},
  journal={IEEE Transactions on Pattern Analysis and Machine Intelligence}, 
  title={Auto-Pytorch: Multi-Fidelity MetaLearning for Efficient and Robust AutoDL}, 
  year={2021},}

@inproceedings{vakili2022improved,
  title={Improved convergence rates for sparse approximation methods in kernel-based learning},
  author={Vakili, Sattar and Scarlett, Jonathan and Shiu, Da-shan and Bernacchia, Alberto},
  booktitle={International Conference on Machine Learning},
  year={2022},
}

@inproceedings{astudillo2023qeubo,
  title={qEUBO: A decision-theoretic acquisition function for preferential Bayesian optimization},
  author={Astudillo, Raul and Lin, Zhiyuan Jerry and Bakshy, Eytan and Frazier, Peter},
  booktitle={International conference on artificial intelligence and statistics},
  year={2023}
}
